\documentclass[sigconf,nonacm]{acmart}
\setcopyright{none}

\renewcommand\footnotetextcopyrightpermission[1]{}
\renewcommand{\shortauthors}{Jiang et al.}
\AtBeginDocument{%
  }

\usepackage{thmtools}
\usepackage{thm-restate}
\usepackage{subcaption}
\usepackage[utf8]{inputenc} 
\usepackage[T1]{fontenc}    
\hypersetup{
  colorlinks=true,
  linkcolor=red,
  citecolor=blue,
  filecolor=blue,
  urlcolor=blue
}

\usepackage{url}

\usepackage{amsthm}
\usepackage{mathtools}
\usepackage[ruled,vlined]{algorithm2e}
\usepackage{algpseudocode}
\usepackage[x11names,svgnames,dvipsnames,table]{xcolor}
\usepackage{subcaption}
\usepackage{thmtools, thm-restate}
\usepackage{cleveref} 
\usepackage{algpseudocode}
\usepackage{multirow}
\usepackage{enumitem}

\usepackage{pifont}    
\usepackage{booktabs}
\usepackage{wrapfig} 
\usepackage{setspace} 
\usepackage{booktabs}   
\usepackage{tabularx}   
\usepackage{ragged2e}   
\usepackage{xstring}  
\usepackage{collcell} 
\usepackage{graphicx}
\usepackage{caption}
\usepackage{amsthm}
\newtheorem{remark}{Remark}
\usepackage{enumitem}
\usepackage{booktabs,multirow}
\usepackage{siunitx}
\usepackage{adjustbox}

\newcommand{\cmark}{\ding{51}} 
\newcommand{\xmark}{\ding{55}} 

\newcommand{\xb}{\mathbf{x}}
\newcommand{\ub}{\mathbf{u}}
\newcommand{\zb}{\mathbf{z}}
\newcommand{\mub}{\boldsymbol{\mu}}
\newcommand{\ellb}{\boldsymbol{\ell}}

\newtheorem{assumption}{Assumption}[section]
\newtheorem{definition}{Definition}[section]
\newtheorem{theorem}{Theorem}[section]

\definecolor{yellow}{HTML}{FFD966} 
\definecolor{perturbation}{HTML}{EA9A90}
\definecolor{obs}{HTML}{7789B7}

\crefname{section}{\S}{\S\S}
\crefname{equation}{Eq.}{Equations}
\crefname{appendix}{App.}{Apps.}
\crefname{thm}{Thm.}{Thms.}
\crefname{cor}{Cor.}{Cors.}
\crefname{prop}{Prop.}{Props.}
\crefname{asm}{Asm.}{Asms.}
\crefname{defn}{Defn.}{Defns.}
\crefname{lemma}{Lem.}{Lems.}
\crefname{exm}{Ex.}{Exs.}
\crefname{clar}{Clar.}{Clars.}

\begin{document}



\title{Learning Latent Dynamical Causal Processes for Single-Cell Perturbation Prediction}
\author{Wenkang Jiang}
\email{wenkang.jiang@adelaide.edu.au}
\affiliation{%
  \institution{AIML, Adelaide University}
  \city{Adelaide}
  \country{Australia}
}

\author{Yuhang Liu}
\authornote{Corresponding author.}
\email{yuhang.liu01@adelaide.edu.au}
\affiliation{%
\institution{Responsible AI Research Centre}
  \institution{AIML, Adelaide University}
  \city{Adelaide}
  \country{Australia}
}

\author{Erdun Gao}
\email{erdun.gao@adelaide.edu.au}
\affiliation{%
\institution{Responsible AI Research Centre}
  \institution{AIML, Adelaide University}
  \city{Adelaide}
  \country{Australia}
}

\author{Ehsan Abbasnejad}
\email{ehsan.abbasnejad@monash.edu}
\affiliation{%
  \institution{Monash University}
  \city{Melbourne}
  \country{Australia}
}

\author{Lina Yao}
\email{lina.yao@unsw.edu.au}
\affiliation{%
  \institution{University of New South Wales}
  \institution{CSIRO’s Data 61}
  \city{Sydney}
  \country{Australia}
}

\author{Javen Qinfeng Shi}
\email{javen.shi@adelaide.edu.au}
\affiliation{%
\institution{Responsible AI Research Centre}
  \institution{AIML, Adelaide University}
  \city{Adelaide}
  \country{Australia}
}

\renewcommand{\shortauthors}{Jiang et al.}

\begin{abstract}
Single-cell perturbation prediction aims to predict how cells would respond to unseen perturbations and achieve out-of-distribution (OOD) generalization, with the goal of understanding how perturbations reshape underlying cellular programs over time. While recent machine learning methods have made progress on this task, most existing approaches focus on only one aspect of perturbation responses. Some methods emphasize learning latent causal mechanisms for generalization, but treat responses as static outcomes and ignore how gene expression evolves over time. Other methods focus on modeling temporal changes in gene expression, but do not explicitly capture the underlying causal generative mechanisms that drive these changes. In reality, however, perturbation effects exhibit two inseparable characteristics: they are \textit{latent}, in that perturbations act through unobserved cellular programs, and \textit{dynamical}, in that the states of these programs evolve over time to generate the observed gene expression profiles. Motivated by this observation, we propose a latent dynamical causal generative model for single-cell perturbation data that jointly captures latent cellular programs and their temporal evolution. We further provide a theoretical analysis showing that, under suitable conditions, the latent causal variables are recoverable up to a trivial equivalence class. Guided by this identifiability analysis, we develop a learning framework for recovering \emph{latent cellular programs and their dynamical evolution} from
single-cell sequencing data. Experiments on Causal-3DIdent, a widely adopted benchmark in causal representation learning, validate the theoretical results and the effectiveness of the proposed method. Additional experiments on real-world CRISPR-based single-cell sequencing perturbation data demonstrate improved generalization to unseen perturbations compared to state-of-the-art baselines, highlighting the robustness of the proposed method.
\end{abstract} 




\maketitle

\section{Introduction}
\label{sec:introduction}

Single-cell perturbation prediction has emerged as a central computational task in the analysis of high-throughput perturbation screens, which aim to model how genetic or chemical interventions reshape cellular states at single-cell resolution. The advent of technologies such as Perturb-seq and related CRISPR-based pooled assays has enabled the collection of large-scale, high-dimensional single-cell perturbation datasets, making it possible to train machine learning models that predict perturbation responses, with important downstream implications for guiding experimental design in functional genomics and drug discovery \citep{norman2019exploring,replogle2020combinatorial}. A central challenge in this task is to reliably infer cellular responses under unmeasured or unseen perturbations, for example, generalizing from models trained on single-perturbation data to unseen perturbation settings~~\citep{tejada2025causal}.

\textbf{Existing Work.} Existing machine learning approaches to single-cell perturbation prediction broadly fall into two distinct directions: leveraging latent causal models and leveraging temporal information (See~\Cref{app:relatedwork} for a detailed discussion). Latent causal models aim to learn latent causal mechanisms that support generalization and interpretability \citep{yang2021causalvae, lachapelle2022disentanglement,zhang2023identifiability,lopez2022learning,de2025interpretable,liu2022identifying,liu2023identifiable,liu2024towards,liu2025latent,liu2026beyond,liu2026i}. These methods typically employ latent causal generative models to capture high-level variables governing gene expression and recover them from observational and perturbed data, a paradigm commonly referred to as causal representation learning \citep{scholkopf2021toward}. Temporal approaches focus on modeling how gene expression changes over time from measurements collected at multiple discrete timepoints, including gene regulatory network inference and ODE-based dynamical systems \citep{ishikawa2023renge,von2025learning,zhang2024scnode}. Such models often operate directly in the observed gene expression space, or introduce latent variables primarily to facilitate dynamical fitting only, without an emphasis on latent generative processes.

\textbf{Challenges.} However, both lines above overlook a key aspect of single-cell perturbation responses: the effect of a perturbation unfolds as a \textit{latent dynamical generative process}, characterized by two inseparable aspects: \textit{latent} and \textit{dynamical}. In practice, \textit{Latent} refers to the fact that perturbations do not act directly on observed gene expression, but instead operate through latent generative processes that capture how genetic or chemical interventions influence expression patterns \citep{scholkopf2021toward,lachapelle2022disentanglement,zhang2023identifiability}. \textit{Dynamical} reflects that the states of these latent processes evolve over time following the perturbation, with early changes propagating through regulatory processes and giving rise to stage-specific gene expression profiles observed at different timepoints \citep{ishikawa2023renge,zhang2024scnode,von2025learning, jiang2026makes}. 

\textbf{Implications.} As a consequence, existing approaches tend to capture only one aspect of this generative process at a time. Specifically, methods based on latent causal modeling \citep{lachapelle2022disentanglement,zhang2023identifiability,lopez2022learning,de2025interpretable} emphasize recovering interpretable latent mechanisms, but typically abstract away the temporal evolution through which perturbation effects unfold, treating responses as static or snapshot-based outcomes. Conversely, temporal modeling approaches \citep{ishikawa2023renge,von2025learning,zhang2024scnode} focus on fitting how gene expression changes over time, but often operate in the observed space or employ latents primarily for dynamical fitting, without enforcing that these latent states correspond to stable, causally meaningful cellular programs.

\textbf{Contributions.} We move beyond approaches that model either latent mechanisms or temporal dynamics in isolation, and instead propose a unified framework that captures both aspects simultaneously. This paper makes the following contributions:

\begin{itemize}[leftmargin=5pt]
    \item \textit{A latent dynamical generative model.} We propose a latent dynamical causal generative model for single-cell perturbation responses, in which a latent causal model governs data generation at each timepoint, together with a causal dynamical process that captures temporal evolution across time (\Cref{sec:setup}).
    \item \textit{Theoretical Identifiability.} We provide a rigorous theoretical analysis showing that, under certain assumption, the latent causal variables in proposed latent dynamical generative model can uniquely recovered up to a trivial transformation (~\Cref{subsec:target_id}). 
    \item \textit{A Theory-Motivated Method.} Guided by the identifiability analysis, we develop \textsc{CITE-VAE}, a Causality-Inspired temporal Variational AutoEncoder framework that learns causally meaningful latent dynamics, thereby enabling principled generalization grounded in causal mechanisms.
    (~\Cref{sec:method}).
    \item \textit{Empirical Validation.} We evaluate the proposed method on synthetic data to verify our theoretical results under controlled settings where the assumptions hold, and further demonstrate that it significantly outperforms state-of-the-art baselines on real-world perturbation datasets, where the assumptions might not strictly hold, highlighting the practicality of the proposed method (~\Cref{sec:experiments}).
\end{itemize}

\section{Related Work}
\label{app:relatedwork}

\subsection{Causal Dynamics in bioinformatics.}
Modeling single-cell dynamics from snapshot data often relies on optimal transport (OT). OT-based trajectory reconstruction is accurate but suffers from scalability and multi-stage optimization. NODE-based alternatives remove OT but often sacrifice interpretability. Cell-MNN \citep{von2025learning} addresses this by modeling latent dynamics through a locally linearized ODE, yielding an OT-free, end-to-end architecture that scales to large datasets and recovers interpretable GRN structure. Complementing these latent approaches, chronODE \citep{borsari2025chronode} fits observable multi-omic trajectories using a biophysically motivated logistic ODE, revealing a fundamental rate–saturation trade-off in gene regulation and showing how proximal versus distal regulatory elements influence expression direction and temporal speed.

A large body of prior work on temporal biological data focuses on structure learning from observed time series, predominantly using variants of Granger causality to infer directed edges between measured variables~\citep{wu2025unveiling}, and extending these principles to adapt the framework to graph-based partial orderings constructed from RNA velocity estimates~\citep{singh2024causal}, notably partial orderings inherent in branching differentiation trajectories~\citep{tejada2025causal}, MINIE~\citep{moscardo2025multi} couples differential--algebraic equations with latent pseudotime, yet relies on linear mappings over high-dimensional measurements, limiting its ability to capture non-linear latent drivers. These methods operate entirely in the observation space, estimating regulatory influence directly from expression trajectories or pseudotime-ordered profiles. While effective for correlation-based network reconstruction, they do not learn a latent representation, and therefore cannot capture unobserved causal factors, invariant mechanisms, or modular structure underlying the dynamics.

Furthermore, HALO~\citep{mao2025halo} introduces VAEs into multi-omic modeling, but its causal reasoning is still fundamentally Granger-based, operating on hand-partitioned latent dimensions rather than identifiable causal factors. These approaches~\citep{lotfollahi2023predicting, mao2025halo} therefore fall short of CRL, as they do not identify latent causal variables or their dynamics. A general CRL framework capable of discovering latent causal mechanisms from perturbation data remains an open challenge.

\subsection{Identifiable causal representations.} A key aim in modeling complex systems is to learn low-dimensional latent variables $\mathbf{z}$ from high-dimensional data $\mathbf{x}$ that match the true generative factors (independent components)~\citep{hyvarinen2001independent}. Nonlinear ICA showed that such components are not identifiable from i.i.d.\ data without extra assumptions~\citep{hyvarinen1999nonlinear}. Identifiable variants address this by introducing an auxiliary variable $\mathbf{u}$ so that latent factors $\{z_i\}_{i=1}^p$ are conditionally independent given $\mathbf{u}$~\citep{hyvarinen2016unsupervised,hyvarinen2017nonlinear}. The iVAE framework~\citep{khemakhem2020variational}, built on VAEs~\citep{kingma2013auto,rezende2014stochastic}, proves identifiability of both $\mathbf{z}$ and $p(\mathbf{x}\mid \mathbf{z})$ under mild conditions. Recent approaches impose structure in latent space: DAG-based models enforce acyclicity~\citep{yang2021causalvae,liu2022identifying,liu2023identifiable,ahuja2023interventional}, while factorized designs split latent variables into invariant, intervention-specific, and interaction parts~\citep{von2021self,gao2025domain}. While prior methods establish identifiability via auxiliary conditioning or broad structural constraints, our model ties perturbations directly to latent mechanisms. This design moves beyond heuristic augmentations or globally factorized latents, making our framework specifically tailored to single-cell perturbation.

\section{Setup and Latent Dynamical Causal Model}
\label{sec:setup}
In this section, we formalize the single-cell perturbation prediction problem and
introduce a latent dynamical causal generative model that serves as the foundation
for the subsequent theoretical analysis in~\Cref{subsec:target_id} and methodological developments in ~\Cref{sec:method}. We first define the problem
setting under unpaired snapshot observations in~\Cref{sec:problem_setup}. We then propose a latent causal  generative process that models how perturbations influence the temporal evolution of unobserved cellular states and give rise to the observed gene expression profiles in~\Cref{sec:dgp}. 

\subsection{Problem Setup}
\label{sec:problem_setup}
We consider single-cell perturbation experiments in which genetic or chemical perturbations are applied to a population of cells and gene expression is measured at multiple timepoints. Due to experimental constraints, measurements are collected at a finite set of discrete snapshot times $t \in \{0,1,\cdots,\,T\}$.
At each snapshot time $t$, we observe gene expression profiles
$\mathbf{x}^t \in \mathbb{R}^p$ from a population of cells, where $p$ denotes the number of measured genes. Each observed gene expression profile $\mathbf{x}^t$ is generated from an underlying,
unobserved cellular state $\mathbf{z}^t \in \mathbb{R}^d$.
This latent state represents the internal biological condition of a cell at time $t$, such as regulatory activity, pathway activation, or other abstract cellular programs, which are not directly observable but give rise to the measured gene expression through an unknown generative mechanism.

Each experiment is conducted under a perturbation condition indexed by a label
$\mathbf u \in \mathcal U$, which specifies the applied genetic or chemical
intervention. The perturbation is defined at the condition level and remains fixed
throughout the experiment; consequently, the same intervention label $\mathbf u$
is shared across all snapshot times $t$ within a given condition. We do not assume
access to the underlying biochemical or molecular mechanism of the intervention.
Instead, $\mathbf u$ is treated as an abstract intervention index that identifies
which perturbation has been applied.

Given observations $\{\mathbf{x}^t, \mathbf{u}\}$ collected under a finite set of
perturbation conditions, our goal is to learn a predictive model that can infer
cellular responses at future timepoints and under previously unseen perturbations.
Concretely, we aim to predict the observed distribution under unseen perturbations, as following:
\begin{equation}
p(\mathbf{x}^{t'} \mid \mathbf{u}), \qquad
\text{for } t' > t \text{ and unseen perturbations}.
\end{equation}
Since single-cell measurements at different timepoints do not correspond to the
same cells, the data consist of population-level snapshots rather than individual cells. As a result, achieving this goal requires modeling how the
latent cellular states $\mathbf{z}^t$ underlying these snapshots evolve over time under different perturbation conditions.

\subsection{A Latent Dynamical Causal Model}
\label{sec:dgp}
We now propose a latent dynamical causal generative model that captures how perturbations affect the temporal evolution of unobserved cellular states and generate the observed gene expression.

We formalize single-cell perturbation data using a latent dynamical causal generative model. To this end, motivated by limited-perturbation settings commonly encountered in single-cell experiments, we decompose the latent state $\mathbf{z}^t$ into two components, i.e., $\mathbf{z}^t = (\mathbf{z}_{\iota}^t, \mathbf{z}_{\nu}^t)$, where $\mathbf{z}_{\iota}^t \in \mathbb{R}^{d_\iota}$ represents perturbation-invariant background programs, and $\mathbf{z}_{\nu}^t \in \mathbb{R}^{d_\nu}$ denotes perturbation-responsive mechanisms. Following standard causal modeling assumptions, we associate each latent state with an independent exogenous variable, i.e., $\mathbf{n}_\iota$ for $\mathbf{z}_\iota$ and $\mathbf{n}_{\nu,i}$ for each component of $\mathbf{z}_{\nu,i}$, capturing sources of information. Given the above, under a perturbation condition $\mathbf{u}$,
we propose the following latent dynamical causal model (Figure~\ref{fig:causalgraph}):
\begin{align}
&p(\mathbf{n}_{\iota}^{t})
:= \mathcal N\!\big(\boldsymbol\mu_\iota,
\boldsymbol\sigma_\iota^2\big), \quad   
p(\mathbf{n}_{\nu}^{t} \mid \mathbf{u})=\mathcal N\!\big(\boldsymbol\mu_\nu(\mathbf u),\boldsymbol\sigma_\nu^2(\mathbf u)\big),
\label{eq:dgp_noise_nu}\\
&\mathbf{z}_{\iota}^{t}
:= \mathbf{f}_\iota(\mathbf{z}_{\iota}^{t-1}) + \mathbf{n}_{\iota}^{t},  \quad
z_{\nu,j}^{t}= {f}_{\nu,j}\!\big(\mathbf{z}_{\mathrm{pa}(j)}^{t-1},\mathbf{u}\big)
+ n_{\nu,j}^{t},
\label{eq:dgp_znu}\\
&\mathbf{x}^t
:= \mathbf{g}(\mathbf{z}^t).
\label{eq:dgp_x}
\end{align}

\begin{itemize}[leftmargin=0pt]
\item \textbf{Latent Exogenous Models Eq.~\eqref{eq:dgp_noise_nu}.}
Here, the perturbation-invariant exogenous variables $\mathbf{n}_{\iota}^{t}$ are assumed to be independent, while the perturbation-responsive exogenous variables $\mathbf{n}_{\nu}^{t}$ are assumed to be independent conditional on the perturbation label $\mathbf{u}$, which means that the distribution of $\mathbf{n}_{\nu}^{t}$ is allowed to vary with $\mathbf{u}$, capturing perturbation-induced changes in the exogenous innovations. Both $\mathbf{n}_{\iota}^{t}$ and $\mathbf{n}_{\nu}^{t}$ are modeled as Gaussian
random variables, with mean $\boldsymbol\mu_\iota$ and diagonal covariance
$\boldsymbol\sigma_\iota^2$, and mean $\boldsymbol\mu_\nu(\mathbf u)$ and diagonal
covariance $\boldsymbol\sigma_\nu^2(\mathbf u)$, respectively.

\item \textbf{Latent Dynamics Eq.~\eqref{eq:dgp_znu}.}
The perturbation-invariant latent states $\mathbf{z}_{\iota}^{t}$ evolve according to a $\mathbf{u}$-independent transition function $\mathbf{f}_\iota$, modeling background dynamics that are unaffected by perturbations. For the perturbation-responsive mechanisms, each latent $z_{\nu,j}^{t+1}$ follows a structural causal equation where the parent nodes $\mathbf{z}_{\mathrm{pa}(j)}^{t-1} \subseteq \mathbf{z}^{t-1}$ may include both perturbation-invariant and perturbation-responsive latent variables, by the mapping ${f}_{\nu,j}$, together with the perturbation label $\mathbf{u}$, for $j = 1,\ldots,d_\nu$. All causal dependencies are lagged from time $t-1$ to $t$, so the temporal graph is acyclic by construction.

\item \textbf{Observation Model Eq.~\eqref{eq:dgp_x}.}
The observed gene expression profile is generated from the latent state
via a time-invariant nonlinear mapping $\mathbf{g}$. We assume $\mathbf{g}$ is smooth and invertible, which is aligned with most of works on causal representation learning~\citep{lachapelle2022disentanglement,zhang2023identifiability,lopez2022learning,de2025interpretable}.
\end{itemize}

\begin{figure}[h]
\centering
\includegraphics[width=0.6\linewidth]{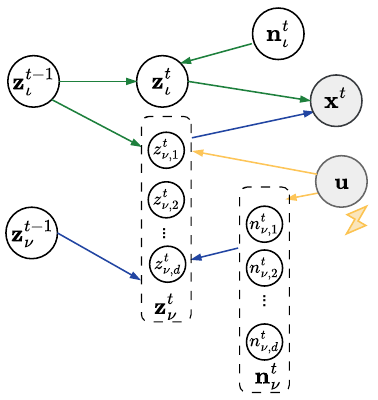}
\caption{\textbf{Illustration of data generation process with the proposed latent dynamical causal model.}
Observations $\mathbf{x}^t$ are generated from time-invariant factors $\mathbf{z}_\iota^t$ and time-variant factors $\mathbf{z}^t_\nu$.
$\mathbf{z}^t_\iota$ conditions the transition dynamics of $\mathbf{z}^{t-1}_\nu$ and $\mathbf{z}^{t-1}_\iota$, while interventions $\mathbf{u}$ induce causal mechanism shifts.}
\label{fig:causalgraph}
\end{figure}


\section{Theoretical Guarantee: Identifiability}
\label{subsec:target_id}

Our aim is to establish \emph{identifiability} for the proposed latent dynamic causal model, i.e., we seek conditions under which the latent variables $\mathbf{z}^t$ can be uniquely recovered, up to a trivial transformation, from observed temporal gene expression snapshots $\{\mathbf{x}^t\}_{t=0}^{T}$ under different perturbations $\mathbf{u}$. Such identifiability guarantees that the learned model captures the underlying causal mechanisms, supporting reliable generalization to unseen perturbations.

Without additional assumptions, exactly recovering latent variables is in general impossible from the observed variables alone, e.g., $\mathbf{x}$ and $\mathbf{u}$, even in the relatively simple nonlinear ICA setting~\cite{hyvarinen1999nonlinear,hyvarinen2016unsupervised,hyvarinen2017nonlinear}. To enable identifiability analysis that follows, we introduce the following assumptions.

\begin{assumption}[Environmental Sufficiency]
\label{asm:ts_rank}
There exist $2d_\nu$ distinct environments $\{\mathbf u_1,\ldots,\mathbf u_{2d_\nu}\}$ relative to a reference $\mathbf u_0$
such that the matrix
\begin{equation}
\mathbf L^\top =
\big[\Delta\boldsymbol\eta(\mathbf u_1),\ldots,\Delta\boldsymbol\eta(\mathbf u_{2d_\nu})\big]^\top
\in\mathbb R^{2d_\nu\times 2d_\nu}
\end{equation}
has full rank $2d_\nu$, where
\begin{equation}
\Delta\boldsymbol{\eta}(\mathbf{u})
:=
\begin{pmatrix}
\frac{\boldsymbol{\mu}_\nu(\mathbf{u})}{\boldsymbol{\sigma}_\nu^2(\mathbf{u})}
-\frac{\boldsymbol{\mu}_\nu(\mathbf{u}_0)}{\boldsymbol{\sigma}_\nu^2(\mathbf{u}_0)}\\[2mm]
-\frac12\Big(\frac{1}{\boldsymbol{\sigma}_\nu^2(\mathbf{u})}-\frac{1}{\boldsymbol{\sigma}_\nu^2(\mathbf{u}_0)}\Big)
\end{pmatrix}\in\mathbb{R}^{2d_\nu},
\end{equation}
with elementwise divisions.
\end{assumption}

\begin{assumption}[Optimal Temporal Alignment]
\label{asm:ts_align}
Let $\mathbf{\hat g}$ be an estimated decoder and $\mathbf{\hat f}=\mathbf{\hat g}^{-1}$ the corresponding encoder,
with $\mathbf{\hat f}(\mathbf x^t)=(\mathbf{\hat f}_\iota(\mathbf x^t),\mathbf{\hat f}_\nu(\mathbf x^t))$. For any fixed $t$ and invariant state $\mathbf z_\iota^t$, let
$\mathbf x^{t,(\mathbf u)}:=\mathbf{g}(\mathbf z^{t,(\mathbf u)})$ and
$\mathbf x^{t,(\mathbf u_0)}:=\mathbf{g}(\mathbf z^{t,(\mathbf u_0)})$ be two observations generated
with the same $\mathbf z_\iota^t$ but independent draws of the responsive variables
(or innovations) under environments $\mathbf u$ and $\mathbf u_0$, respectively. Define the alignment loss
\begin{equation}
\mathcal L_{\mathrm{align}}
:=
\mathbb E\Big[
\big\|
\mathbf{\hat f}_\iota(\mathbf x^{t,(\mathbf u)})
-
\mathbf{\hat f}_\iota(\mathbf x^{t,(\mathbf u_0)})
\big\|_2^2
\Big],
\label{eq:ts_align_loss}
\end{equation}
where the expectation is taken over $\mathbf z_\iota^t$ and the responsive randomness under the above coupling. Assume $\mathcal L_{\mathrm{align}}$ achieves its population global minimum, i.e., $\mathcal L_{\mathrm{align}}=0$. Then $\mathbf{\hat f}_\iota(\mathbf x^t)$ is $\mathbf u$-invariant almost surely and depends only on
$\mathbf z_\iota^t$, i.e., there exists a measurable map $\mathbf{h}_\iota$ such that
$\mathbf{\hat f}_\iota(\mathbf x^t)=\mathbf{h}_\iota(\mathbf z_\iota^t)$ almost surely.
\end{assumption}

\begin{assumption}[Temporal Intervention Sufficiency]
\label{asm:ts_isolation}
For each responsive coordinate $j\in\{1,\ldots,d_\nu\}$, there exists an environment $\mathbf u^{(j)}$
such that the parent contribution to $z_{\nu,j}^{t}$ is removed, i.e.,
\begin{equation}
f_{\nu,j}(\mathbf z_{\mathrm{pa}(j)}^{t-1},\mathbf u^{(j)}) \equiv c_j(\mathbf u^{(j)}) \quad \text{(a constant w.r.t. $\mathbf z^t$)}.
\label{eq:ts_isolation}
\end{equation}
Hence, under $\mathbf u^{(j)}$, we have $z_{\nu,j}^{t}=c_j(\mathbf u^{(j)})+n_{\nu,j}^{t}$.
\end{assumption}

\begin{theorem}[Temporal Identifiability of Latent States]
\label{thm:ts_identifiability}
Suppose the data are generated by the generative model in
Eqs.~\eqref{eq:dgp_noise_nu}--\eqref{eq:dgp_x} under Assumptions
\ref{asm:ts_rank}--\ref{asm:ts_isolation}.
Let $\boldsymbol{\theta}=(\mathbf{g},\mathbf{f}_\iota,\{f_{\nu,j}\}_j,\boldsymbol{\mu}_\iota,\boldsymbol{\sigma}_\iota^2,\boldsymbol{\mu}_\nu(\cdot),\boldsymbol{\sigma}_\nu^2(\cdot))$
denote the true parameters. Let $\boldsymbol{\hat\theta}$ be another parameter set from the same model class. Assume that for all $t$ and all environments $\mathbf u$, the induced conditional distributions match:
\begin{equation}
p_{\boldsymbol{\theta}}(\mathbf x^{t}\mid \mathbf x^{t-1},\mathbf u)\ \equiv\ p_{\boldsymbol{\hat\theta}}(\mathbf x^{t}\mid \mathbf x^{t-1},\mathbf u).
\label{eq:ts_match_transition}
\end{equation}
Then there exist:
(i) a block nonsingular matrix $\mathbf A_\iota\in\mathbb R^{d_\iota\times d_\iota}$ and a vector $\mathbf c_\iota$, and (ii) a scaling-permutation matrix $\mathbf P_\nu$ and a vector $\mathbf c_\nu$,
such that for all $t$,
\begin{align}
\mathbf z_\iota^t &= \mathbf A_\iota\,\hat{\mathbf z}_\iota^t + \mathbf c_\iota,
\label{eq:ts_id_iota}\\
\mathbf z_\nu^t &= \mathbf P_\nu\,\hat{\mathbf z}_\nu^t + \mathbf c_\nu.
\label{eq:ts_id_nu}
\end{align}
That is, the invariant state block is identifiable up to an invertible linear block transform,
and the responsive state block is identifiable up to permutation and component-wise scaling.
\end{theorem}
Proof can be found in~\Cref{app:proof_identifiability}.
\hfill$\square$
\paragraph{Proof sketch.}
The proof proceeds in three steps. In Step I, Using invertibility of the observation map, transition matching between $(\mathbf x^{t-1},\mathbf x^{t})$ is reduced to matching latent innovations $\mathbf n^t$ via a change of variables with a triangular Jacobian, eliminating environment-invariant factors. In Step II, using environment-dependent Gaussian conditionals and a rank condition, the perturbation-responsive innovations $\mathbf n_\nu^t$ are identified component-wise up to scaling and permutation by nonlinear-ICA arguments, while the invariant innovations $\mathbf n_\iota^t$ are identified up to a nonsingular linear transform using the alignment constraint. In Step III, Using the additive transition structure and innovation identifiability, innovation identifiability is lifted to state identifiability, yielding scaling--permutation identifiability for $\mathbf z_\nu^t$ and linear identifiability for $\mathbf z_\iota^t$.

\paragraph{Justification of Assumptions~\ref{asm:ts_rank}-\ref{asm:ts_isolation}.}
Assumption~\ref{asm:ts_rank} requires that the environment or
intervention label induces sufficiently rich and diverse changes in the distributions
of the responsive innovations, so that different latent factors exhibit distinct
distributional responses and can be disentangled component-wise. This type of environment diversity condition is standard in nonlinear ICA and
identifiable latent variable models with auxiliary variables
\citep{hyvarinen2016unsupervised,khemakhem2020variational,liu2022identifying}. Assumption~\ref{asm:ts_align}  enforces that, at the population
optimum of the alignment objective, representations that mix invariant and
environment-dependent information are ruled out, forcing the learned invariant block
to depend only on the perturbation-invariant latent state.
Such population-level optimality assumptions are commonly adopted in theoretical
analyses of contrastive and invariant representation learning
\citep{von2021self}. Assumption~\ref{asm:ts_isolation} ensures that, for
each responsive latent coordinate, there exists at least one environment in which
parent contributions are removed, thereby isolating the corresponding exogenous
innovation and preventing degenerate causal structures.
Similar isolation or hard-intervention conditions are widely used in identifiability
analyses of causal representation learning and dynamical causal models
\citep{liu2022identifying,liu2023identifiable,lippe2022causal,von2023nonparametric}.

\textit{Insights.} The above theoretical analysis provides a principled explanation of why combining likelihood-based learning with alignment objectives can, in principle, recover uniquely latent causal variables in temporal settings. Maximizing the likelihood enforces that the learned latent variables explain the observed data through a valid generative process, while alignment across environments eliminates spurious degrees of freedom by preventing invariant and responsive factors from being entangled. Together, these two ingredients are sufficient to identify the true latent structure up to the unavoidable equivalence classes.

Importantly, the analysis reveals a clear division of labor between the two objectives. Likelihood maximization recovers the latent variables only up to broad affine or permutation ambiguities, whereas alignment provides the additional constraints needed to pin down invariant subspaces and isolate perturbation-specific mechanisms. This interplay explains why neither objective alone is sufficient, but their combination yields identifiability guarantees under mild and interpretable assumptions. Beyond theoretical interest, these results offer concrete guidance for method design. They suggest that effective representation learning in dynamical and perturbed systems should explicitly couple generative modeling with alignment or invariance constraints, rather than relying on either component in isolation. This insight directly motivates the design choices in our proposed approach, where likelihood-based training ensures fidelity to the data-generating process and alignment guide the model toward causal representations.

\section{Methodology: Theory-Driven \textsc{CITE-VAE}}
\label{sec:method}
\paragraph{From identifiability theory to algorithm design.}
The identifiability analysis in~\Cref{subsec:target_id} establishes that, in temporal
interventional settings with single-cell perturbation observations, recovering causally meaningful latent variables requires the \emph{joint} effect of likelihood-based generative modeling and invariant alignment across environments. \textit{Likelihood} alone constrains representations to explain the observed data distribution, but admits large equivalence classes that are not causally interpretable. \textit{Invariant alignment} removes these ambiguities by enforcing block-wise invariance, thereby enabling identifiability of latent causal dynamics. \textsc{CITE-VAE}, a Causality-Inspired temporal Variational
AutoEncoder framework, is designed to faithfully translate these population-level theoretical requirements into a practical learning framework.

\subsection{Likelihood Maximization}
\label{subsec:likelihood}

We start from a likelihood-based formulation of the latent dynamical causal process in Eqs.~\ref{eq:dgp_noise_nu}–\ref{eq:dgp_x}. In principle, maximizing the likelihood amounts to optimizing the following objective:
\begin{equation}
\mathcal L_{\mathrm{lik}}
:=
\mathbb E_{\mathbf u}
\Big[
\mathbb E_{\mathbf x^{0:T}\sim p_{\mathrm{data}}(\cdot\mid\mathbf u)}
\log p_\theta(\mathbf x^{0:T}\mid\mathbf u)
\Big],
\label{eq:likelihood_objective}
\end{equation}
where
\begin{equation}
p_\theta(\mathbf x^{0:T}\mid\mathbf u)
=
\int
\prod_{t=0}^T
p_{\boldsymbol{\theta}}(\mathbf x^t\mid\mathbf z^t)\,
p_{\boldsymbol{\theta}}(\mathbf z^{t}\mid\mathbf z^{t-1},\mathbf u)\,
\mathrm d\mathbf z^{0:T}.
\label{eq:full_likelihood}
\end{equation}
However, the posterior over latent trajectories $p_{\boldsymbol{\theta}}(\mathbf z^{0:T}\mid\mathbf x^{0:T},\mathbf u)$ is intractable in general. We therefore introduce a variational approximation $q(\mathbf z^t\mid\mathbf x^t,\mathbf u)$ and optimize a variational lower bound.
At each timepoint, we impose a structured posterior factorization
\begin{equation}
q(\mathbf z^t\mid\mathbf x^t,\mathbf u)
=
q(\mathbf z_\iota^t\mid\mathbf x^t)
\,
q(\mathbf z_\nu^t\mid\mathbf x^t,\mathbf u),
\label{eq:posterior}
\end{equation}
which reflects the invariant--responsive decomposition, aligned with the proposed latent dynamical causal process in Eqs.~\ref{eq:dgp_noise_nu}–\ref{eq:dgp_x}, as well as Theorem~\ref{thm:ts_identifiability}.

Under this approximation, likelihood maximization is replaced by maximizing a temporal evidence lower bound (ELBO):

\begin{equation}
\begin{aligned}
\mathcal L_{\mathrm{ELBO}}
=
&\mathbb E_{t}
\Big[
\mathbb E_{q(\mathbf z^t\mid\mathbf x^t,\mathbf u)}\log p_\theta(\mathbf x^t\mid\mathbf z^t)
\\
&-
\mathbb E_{q(\mathbf z^{t-1}\mid\mathbf x^{t-1},\mathbf u)}
\mathrm{KL}\!\big(
q(\mathbf z^{t}\mid\mathbf x^{t},\mathbf u)
\;\|\;
p(\mathbf z^{t}\mid\mathbf z^{t-1},\mathbf u)
\big)
\Big].
\end{aligned}
\label{eq:temporal_elbo}
\end{equation}

where the transition prior factorizes according to the assumed causal structure:
\begin{equation}
p(\mathbf z^{t}\mid \mathbf z^{t-1},\mathbf u)
=
p(\mathbf z_\iota^{t}\mid \mathbf z_\iota^{t-1})
\prod_{j=1}^{d_\nu}
p\!\big(
z_{\nu,j}^{t}\mid \mathbf z_{\mathrm{pa}(j)}^{t-1},\mathbf u
\big).
\label{eq:transition_factorization}
\end{equation}

\paragraph{Discussion.}
Maximizing the temporal ELBO in \eqref{eq:temporal_elbo} enforces that the latent
variables retain sufficient information to reconstruct the observed snapshots
and to account for their temporal evolution under interventions.
From an information-theoretic perspective, this objective prevents irreversible
information loss from the observations to the latent space, and thus constitutes
a necessary condition for recovering the underlying latent generative process. However, sufficiency alone does not imply identifiability. Neither the likelihood objective Eq.~\eqref{eq:likelihood_objective} nor its variational relaxation Eq.~\eqref{eq:temporal_elbo} constrains how this information is distributed across latent coordinates. As established in Theorem~\ref{thm:ts_identifiability}, there generally exist multiple latent parameterizations that preserve all information relevant for reconstruction and achieve identical ELBO values, while differing in their internal organization. In particular, perturbation-invariant and perturbation-responsive factors may remain entangled without affecting the variational objective. Crucially, when the information-preserving property induced by ELBO maximization
is combined with an invariant alignment constraint (introduced next), the remaining degrees of freedom are eliminated. Together, likelihood maximization ensures that the latent space captures all relevant information, while alignment enforces the correct partitioning of this information across invariant and responsive subspaces. Under the conditions in Theorem~\ref{thm:ts_identifiability}, this combination is sufficient to guarantee that the latent causal variables $\mathbf z$ are recoverable.

\subsection{Invariant Alignment}
\label{subsec:alignment}

The identifiability analysis in Theorem~\ref{thm:ts_identifiability} shows that likelihood-based learning alone is insufficient for identifying the latent variables. We thus introduce an invariant alignment constraint that enforces a principled separation between perturbation-invariant and perturbation-responsive factors.

For a fixed timepoint $t$, observations generated under different perturbations may differ substantially at the data level, but should share the same invariant latent state $\mathbf z_\iota^t$.
Therefore, representations intended to capture $\mathbf z_\iota^t$ should be
insensitive to the intervention label $\mathbf u$. Formally, let $\hat\pi^{\mathrm{cf}}_{\mathbf u,t}$ denote a coupling between
snapshots observed under intervention $\mathbf u$ and a reference condition
$\mathbf 0$ at the same time $t$.
We introduce the following alignment loss:
\begin{equation}
\mathcal L_{\mathrm{align}}
=
\mathbb E_{(\mathbf x^{t,(\mathbf u)},\mathbf x^{t,(\mathbf 0)})\sim\hat\pi^{\mathrm{cf}}_{\mathbf u,t}}
\Big[
\big\|
\boldsymbol\mu_{\phi,\iota}(\mathbf x^{t,(\mathbf u)})
-
\boldsymbol\mu_{\phi,\iota}(\mathbf x^{t,(\mathbf 0)})
\big\|_2^2
\Big],
\label{eq:align_loss}
\end{equation}
where $\boldsymbol\mu_{\phi,\iota}(\cdot)$ denotes the mean of the variational
posterior $q_{\phi_\iota}(\mathbf z_\iota^t\mid \mathbf x^t)$ introduced in
Eq.~\ref{eq:posterior}. Minimizing $\mathcal L_{\mathrm{align}}$ encourages the inferred invariant representation to be stable across environments.

\paragraph{Discussion.}
Intuitively, the alignment constraint encourages the model to represent perturbation-invariant information in a dedicated latent subspace, by enforcing consistency of the inferred invariant representation across environments. It is not intended to improve reconstruction accuracy, but rather to guide how the information preserved by likelihood maximization is organized within the latent space. From the results in Theorem~\ref{thm:ts_identifiability}, alignment plays a complementary role to ELBO maximization. While the ELBO ensures that no relevant information is lost in the latent representation, the alignment constraint ensure identifying perturbation-invariant information. Together, these two components guarantee that perturbation-invariant and perturbation-responsive factors occupy distinct latent subspaces, yielding
identifiability.

\begin{figure*}[h]
    \centering
    \includegraphics[width=0.85\textwidth]{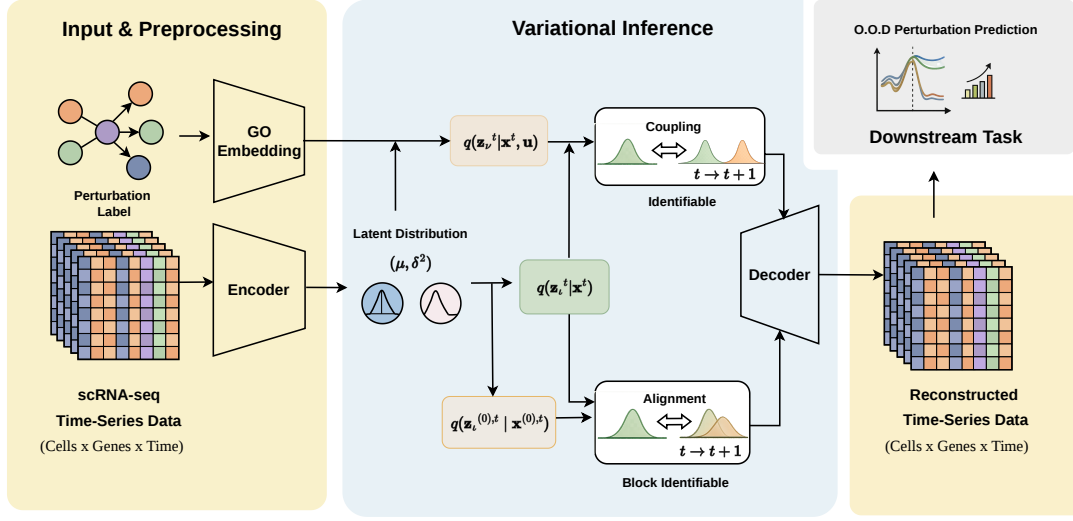}
    \caption{\textsc{CITE-VAE} framework for learning latent causal dynamics from single-cell time-series data. Observed snapshots are encoded into a latent space that is explicitly partitioned into a perturbation-invariant block $\mathbf z_\iota$ and a perturbation-responsive block $\mathbf z_\nu$. Temporal learning is enabled via distributional coupling between consecutive timepoints, allowing latent transitions to be learned without cell-wise pairing. An invariant alignment objective enforces consistency of $\mathbf z_\iota$ across environments, while sparse interventional transition mechanisms govern the evolution of $\mathbf z_\nu$.}
    \label{fig:citevae}
\end{figure*}

\subsection{Sparsity on Latent Causal Structure}
\label{subsec:sparsity}

While the combination of likelihood maximization and invariant alignment provides identifiability guarantees at the population level
(Theorem~\ref{thm:ts_identifiability}), practical learning is inevitably subject to finite-sample effects, estimation noise, and imperfect optimization. In such regimes, the results in the Theorem may not be cleanly resolved, leading to unstable or overly dense latent transition
structures. To mitigate these issues and to bias the solution toward parsimonious and
interpretable causal mechanisms, we introduce an explicit structural sparsity constraint on the responsive dynamics.

We focus on the perturbation-responsive block $\mathbf z_\nu^t$, whose dynamics
encode how perturbations propagate over time. We enforce that each coordinate $z_{\nu,j}^{t}$ depends only on a small subset of parent variables $\mathbf z_{\mathrm{pa}(j)}^{t-1}$. To operationalize this, we introduce a learnable lagged adjacency matrix $\mathbf{A}\in\mathbb R^{d_\nu\times d_\nu}$ that gates cross-coordinate influences
from $\mathbf z_\nu^{t-1}$ to $\mathbf z_\nu^{t}$. To ensure that the induced structural equation model defines a valid directed acyclic graph, we parameterize $A$ as a strictly lower-triangular matrix under a fixed latent ordering. This construction guarantees acyclicity by design, avoiding the need for additional soft constraints or iterative DAG penalties. As a result, we arrive:
\begin{equation}
\mathcal L_{\mathrm{reg}}(\mathbf{A}) = \|\mathbf{A}\|_1.
\label{eq:sparsity_penalty}
\end{equation}
This regularizer suppresses spurious dependencies that may arise due to noise or limited data, and biases the learned dynamics toward sparse and interpretable mechanism changes.

\subsection{Overall Training Objective}
\label{subsec:objective}

We jointly optimize a likelihood-based variational objective with invariant alignment and structural sparsity regularization. In the idealized population setting, likelihood maximization together with alignment yields identifiability of the latent causal variables (\Cref{thm:ts_identifiability}). In practice, finite-sample effects and imperfect optimization can lead to estimation bias and unstable transition structures, motivating the additional sparsity constraint introduced in~\Cref{subsec:sparsity}.

\paragraph{Temporal ELBO}
Given a pseudo-pair $(\mathbf x^{t-1},\mathbf x^{t})\sim \hat\pi_{\mathbf u,t}$ constructed from snapshots, we optimize the
following temporal ELBO, a temporal version of Eq.~\ref{eq:temporal_elbo}:

\begin{equation}
\begin{aligned}
\mathcal L_{\mathrm{temp}}
:=&
\mathbb E_{(\mathbf x^{t-1},\mathbf x^{t})\sim \hat\pi_{\mathbf u,t}}
\Big[
\mathbb E_{q_\phi(\mathbf z^{t-1}\mid \mathbf x^{t-1},\mathbf u)}
\log p_\theta(\mathbf x^{t-1}\mid \mathbf z^{t-1})
\\
&+
\mathbb E_{q_\phi(\mathbf z^{t}\mid \mathbf x^{t},\mathbf u)}
\log p_\theta(\mathbf x^{t}\mid \mathbf z^{t})
\\
&-
\mathbb E_{q_\phi(\mathbf z^{t-1}\mid \mathbf x^{t-1},\mathbf u)}
\mathrm{KL}\!\Big(
q_\phi(\mathbf z^{t}\mid \mathbf x^{t},\mathbf u)
\ \big\|\ 
p_\psi(\mathbf z^{t}\mid \mathbf z^{t-1},\mathbf u)
\Big)
\Big].
\end{aligned}
\label{eq:temp_elbo_final}
\end{equation}

where the transition prior $p_\theta(\mathbf z^{t}\mid \mathbf z^{t-1},\mathbf u)$
factorizes according to the assumed causal structure in
\eqref{eq:transition_factorization}.

Together with the invariant alignment loss $\mathcal L_{\mathrm{align}}$ in
\eqref{eq:align_loss} and the sparsity penalty $\mathcal L_{\mathrm{reg}}(A)$ in
\eqref{eq:sparsity_penalty}, we arrive the following final objective:
\begin{equation}
\max\;
\sum_{\mathbf u}\sum_{t}
\mathcal L_{\mathrm{temp}}
\;-\;
\lambda_{\mathrm{align}}
\sum_{\mathbf u}\sum_{t}
\mathcal L_{\mathrm{align}}
\;-\;
\lambda_{\mathrm{reg}}
\sum_{\mathbf u}\sum_{t}
\mathcal L_{\mathrm{reg}}(A),
\label{eq:final_objective}
\end{equation}
where $\lambda_{\mathrm{align}}$ and $\lambda_{\mathrm{reg}}$ control the strength
of invariant alignment and structural sparsity, respectively. An overview of the proposed \textsc{CITE-VAE} framework is illustrated in \Cref{fig:citevae}.

\section{Empirical Findings}
\label{sec:experiments}
\subsection{Experiments on Causal-3DIdent Data}
To validate our theoretical results under controlled conditions, we first conduct experiments on Causal-3DIdent, a commonly used dataset in causal representation learning for evaluating identifiability~\citep{von2021self, zimmermann2021contrastive}. Specifically, the underlying generative factors include object position and rotation (multi-dimensional), object, background, and spotlight hue, spotlight rotation, and object shape, yielding 11 degrees of variation in total. To extend the original non-temporal dataset to a temporal setting, we impose time-invariant causal dependencies among these factors and generate short sequences via first-order Markov transitions. As illustrated in~\Cref{fig:causal3d_graph}, each latent factor at time $t{+}1$ depends on a subset of parent factors at time $t$, forming a two-time-slice structural equation model. Continuous factors evolve according to Gaussian transition distributions whose means are nonlinear functions of their parents, while a subset of factors (e.g., shape) remains constant within each sequence, providing natural sources
of temporal invariance. Interventions are specified by a binary vector $\mathbf{u}\in\{0,1\}^K$ indicating which latent factors are directly targeted for an entire sequence, with the intervention $\mathbf{u}$ held fixed across time. This construction yields a controlled temporal causal process that aligns with the assumptions of our theoretical analysis. Further details on the factor definitions, transition dynamics, and intervention mechanisms are provided in~\Cref{app:causal3d}.

\begin{figure}[t]
  \centering
  \includegraphics[width=0.95\columnwidth]{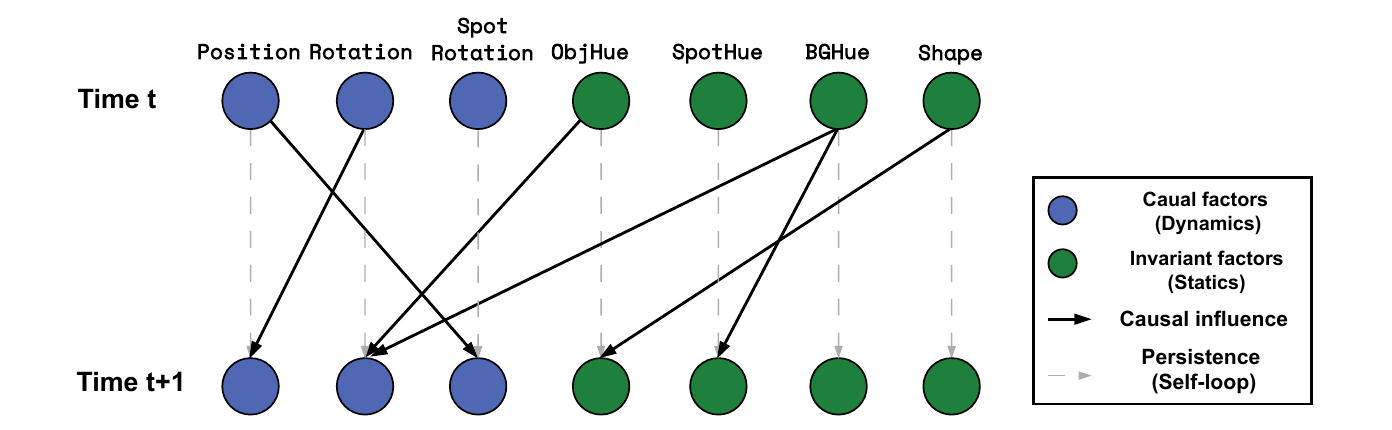}
  \caption{\textbf{Temporal causal structure of the synthetic benchmark.}
  The graph specifies two-time-slice dependencies governing factor transitions from $t$ to $t{+}1$ among the 11 dimensions of variation. Solid arrows denote cross-factor causal influences, while dashed self-loops indicate temporal persistence. Factors are grouped into intervention-responsive (dynamic) and invariant subsets to evaluate identifiability.}
  \label{fig:causal3d_graph}
\end{figure}

\begin{table}[htbp]
\centering
\caption{Non-Linear grounding performance on Causal3DIdent. We report $R^2$ and Spearman correlation for the recovered latent coordinates. The results are consistent with Theorem~\ref{thm:ts_identifiability}.}
\label{tab:causal3dident_grounding}
\begin{tabular}{lcc}
\toprule
\textbf{Representation} & $\mathbf{R^2}$ $\uparrow$ & \textbf{Spearman} $\uparrow$ \\
\midrule
Latent factors $\mathbf{z}_\nu$       & 0.959 & 0.997 \\
Latent factors $\mathbf{z}_\iota$ & 0.998 & 0.988 \\
Global system                        & 0.989 & 0.991 \\

\bottomrule
\end{tabular}
\label{tab:simulation_results}
\end{table}

\paragraph{Results and discussion.} On the Causal3DIdent benchmark, ground-truth latent factors are available, allowing us to directly evaluate identifiability. Following common practice in causal representation learning, we assess how well the learned latent variables recover the true generative factors via nonlinear probing, as reported in \Cref{tab:causal3dident_grounding}.

Specifically, we report explained variance ($R^2$) to measure how accurately the ground-truth factors can be reconstructed from the learned representation using a non-linear map, together with Spearman rank correlation to evaluate monotonic consistency. The latter metric is invariant to invertible reparameterizations and thus well-aligned with the theoretical notion of identifiability. See~\Cref{app:metrics} for more details. Results show that both the perturbation-responsive latents $\mathbf{z}_\nu$ and the invariant latents $\mathbf{z}_\iota$ exhibit strong agreement with the ground-truth factors. This indicates that the learned representation preserves the underlying latent state manifold and recovers the true causal factors, consistent with the guarantees in Theorem~\ref{thm:ts_identifiability}. Additional results are provided in~\Cref{app:causal3d}.

\paragraph{Trajectory  visualization.}
As a qualitative complement to the quantitative grounding results above,
\Cref{fig:manifold_recovery_v} visualizes the temporal trajectories obtained by
linearly projecting the learned perturbation-responsive latents $\mathbf z_\nu$ to the ground-truth Cartesian coordinate space. This visualization is designed to examine whether the learned representation preserves the temporal evolution of the underlying states, rather than only achieving frame-wise factor alignment.

The predicted trajectories closely follow the ground-truth motion across
multiple time steps, and the projected coordinates exhibit strong agreement with
their true counterparts. Notably, the recovered trajectories evolve smoothly over time and maintain the correct relative geometry and direction of motion, indicating that the latent dynamics capture consistent state transitions. This qualitative evidence aligns with the high $R^2$ and rank-based grounding
scores, and illustrates that the learned responsive latents encode temporally
coherent latents that govern the observed dynamics.

\begin{figure}[t]
  \centering
  \begin{subfigure}[t]{0.45\columnwidth}
    \centering
    \includegraphics[width=\linewidth]{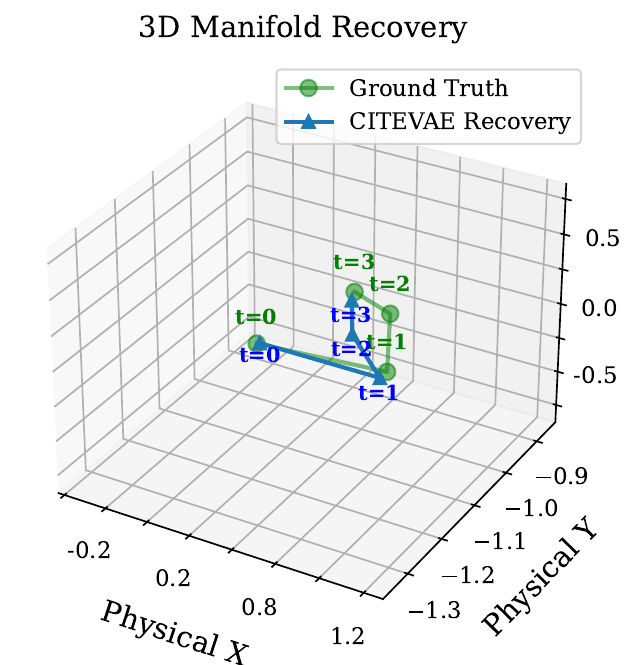}
    \caption{}
    \label{fig:manifold_traj_v}
  \end{subfigure}\hfill
  \begin{subfigure}[t]{0.45\columnwidth}
    \centering
    \includegraphics[width=\linewidth]{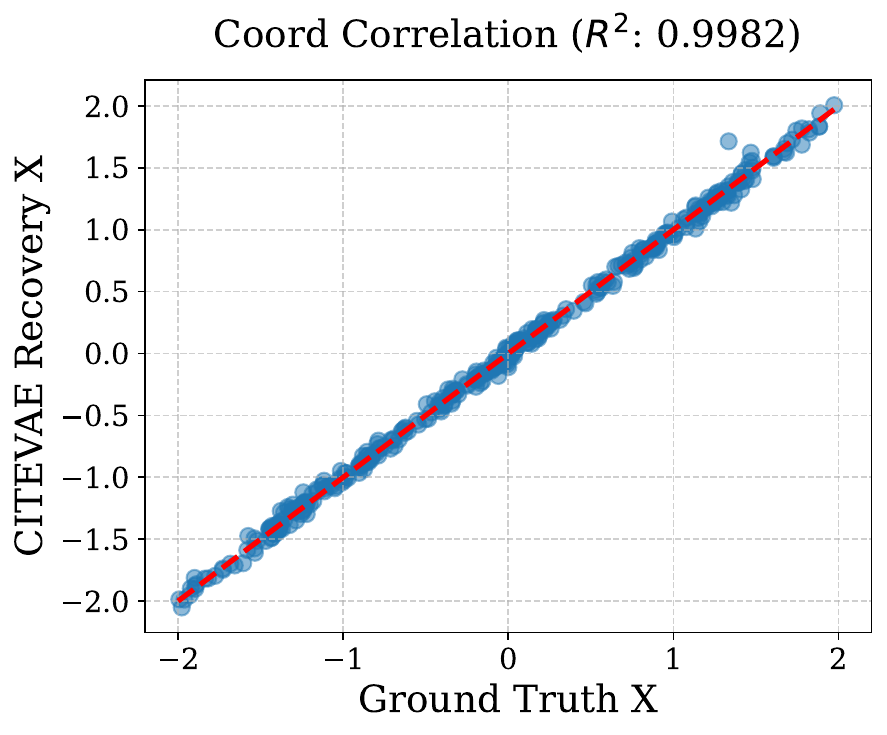}
    \caption{}
    \label{fig:manifold_scatter_v}
  \end{subfigure}

  \caption{\textbf{Trajectory-level visualization from $\mathbf{z}_\nu$.}
The responsive latents $\mathbf{z}_\nu$ are linearly projected to the ground-truth Cartesian coordinates.
\textbf{Left:} the resulting trajectories closely track the ground-truth physical motion over time.
\textbf{Right:} predicted versus true coordinates exhibit strong linear agreement, with points concentrated near the diagonal (high $R^2$).}
\vspace{-1.5em}
\label{fig:manifold_recovery_v}
\end{figure}

\subsection{Experiments on Real-world Single-Cell Data}
Our method is derived from a theory-driven generative model and relies on several assumptions, such as latent additive-noise dynamics. While these assumptions can be explicitly controlled and verified in synthetic settings, they are difficult to strictly validate in real biological systems. We therefore turn to real-world data in this section to assess whether the proposed approach
remains effective when the theoretical assumptions are only approximately satisfied,
and whether the learned representations still exhibit meaningful generalization
and predictive behavior in practice.

To this end, we evaluate our method on a longitudinal hiPSC single-cell CRISPR
perturbation dataset~\citep{ishikawa2023renge}, which profiles human induced pluripotent
stem cell differentiation under 23 distinct single-gene knockouts measured at multiple
discrete time points. The data consists of cross-sectional single-cell RNA-seq snapshots, as continuous individual trajectories are inaccessible due to the destructive nature of sequencing. Given the scarcity of public temporal perturbation benchmarks, we establish a rigorous Leave-One-Intervention-Out (LOO) protocol to assess Out-Of-Distribution (OOD) generalization to unseen gene knockouts. While the LOO protocol enforces discrete OOD at the intervention level, it does not distinguish the functional proximity between held-out and training genes.
We therefore quantify OOD severity by measuring Gene Ontology (GO)~\citep{ashburner2000gene} embedding similarity between the test gene and the training perturbations, inducing a continuous notion of distribution shift.
Furthermore, we also report performance on fold-specific differentially expressed (DE) genes to focus evaluation on perturbation-relevant signals. DE genes are computed within each training fold only, and then applied unchanged to the corresponding test fold, preventing any test-set leakage. See \Cref{app:hispc} for preprocessing details and statistics.

\textit{Baselines.}
We benchmark against a diverse set of baselines designed to reflect the main modeling paradigms for single-cell perturbation prediction. These include latent causal models such as SAMSVAE~\citep{bereket2023modelling} and DiscrepancyVAE~\citep{zhang2023identifiability}, which explicitly posit latent generative structures under interventions; temporal dynamical models, including \textsc{RENGE}~\citep{ishikawa2023renge} and scNode~\citep{zhang2024scnode}, which model state evolution over time but are not necessarily grounded in a causal latent formulation;
large-scale foundation models such as \textsc{UCE}~\citep{rosen2023universal} and \textsc{scGPT}~\citep{cui2024scgpt}, which leverage pretrained representations and temporal conditioning;
and the classical Elastic Net~\citep{zou2005regularization}, serving as a strong linear baseline.
Together, these baselines span a broad spectrum of assumptions regarding temporal conditioning, dynamical structure, and causal modeling, enabling a comprehensive and capability-aware comparison.

\begin{table*}[!t]
\centering
\small
\setlength{\tabcolsep}{5pt}
\renewcommand{\arraystretch}{1.15}
\caption{\textbf{Comprehensive evaluation on the hiPSC temporal CRISPR dataset (LOO mode).}
We first summarize each method's \emph{modeling capabilities} along three dimensions:
\textit{Time} (use of temporal indices),
\textit{Dyn.} (explicit state dynamics),
and \textit{Causal} (latent causal generative modeling).
We then report quantitative performance on \textit{all genes}
(RMSE, $R^2$, MAE) and on \textit{differentially expressed (DE) genes}
($\Delta$ Pearson, AUC-ROC, AUPRC).
For $R^2$, values exceeding $0.95$ are reported as ``$>0.95$'' as they indicate uniformly strong fits
with limited discriminative power.
Entries marked as ``--'' denote degenerate or non-informative metrics.}
\label{tab:hipsc_full_eval}
\begin{adjustbox}{width=0.9\textwidth,center}
\begin{tabular}{lccc|ccc|ccc}
\toprule
Model
& Time & Dyn. & Causal
& \multicolumn{3}{c|}{All genes}
& \multicolumn{3}{c}{DE genes} \\
\cmidrule(lr){2-4}
\cmidrule(lr){5-7}
\cmidrule(lr){8-10}
& & &
& RMSE $\downarrow$ & $R^2 \uparrow$ & MAE $\downarrow$
& $\Delta$ Pearson $\uparrow$ & AUC-ROC $\uparrow$ & AUPRC $\uparrow$ \\
\midrule
\textbf{ElasticNet}~\citep{zou2005regularization}
& \xmark & \xmark & \xmark
& 0.105$\pm$0.053 & $>0.95$ & 0.011$\pm$0.004
& \textemdash & 0.925$\pm$0.014 & 0.625$\pm$0.055 \\

\textbf{scGPT}~\citep{cui2024scgpt}
& \cmark & \xmark & \xmark
& 0.086$\pm$0.068 & $>0.95$ & 0.010$\pm$0.006
& 0.516$\pm$0.022 & 0.919$\pm$0.028 & 0.092$\pm$0.071 \\

\textbf{UCE}~\citep{rosen2023universal}
& \cmark & \xmark & \xmark
& 0.056$\pm$0.024 & $>0.95$ & 0.041$\pm$0.005
& 0.514$\pm$0.023 & 0.933$\pm$0.028 & 0.101$\pm$0.068 \\

\textbf{SAMS-VAE}~\citep{bereket2023modelling}
& \xmark & \xmark & \cmark
& 0.182$\pm$0.017 & $>0.95$ & 0.082$\pm$0.006
& 0.621$\pm$0.017 & 0.941$\pm$0.011 & 0.534$\pm$0.030 \\

\textbf{Discrepancy-VAE}~\citep{zhang2023identifiability}
& \xmark & \xmark & \cmark
& 0.129$\pm$0.046 & $>0.95$ & 0.111$\pm$0.016
& 0.232$\pm$0.017 & 0.622$\pm$0.081 & 0.322$\pm$0.083 \\

\textbf{RENGE}~\citep{ishikawa2023renge}
& \cmark & \cmark & \xmark
& 0.857$\pm$0.313 & \textemdash & \textemdash
& \textemdash & 0.594$\pm$0.054 & 0.970$\pm$0.056 \\

\textbf{scNODE}~\citep{zhang2024scnode}
& \cmark & \cmark & \xmark
& 0.106$\pm$0.053 & 0.876$\pm$0.119 & 0.183$\pm$0.081
& 0.455$\pm$0.016 & \textemdash & \textemdash \\

\textbf{CellOT}~\citep{bunne2023learning}
& \cmark & \cmark & \xmark
& 0.785$\pm$0.062 & $>0.95$ & 0.614$\pm$0.066
& 0.161$\pm$0.191 & 0.785$\pm$0.067 & 0.413$\pm$0.083 \\

\textbf{LFADS}~\citep{pandarinath2018inferring}
& \cmark & \cmark & \xmark
& 0.531$\pm$0.015 & $>0.95$ & 0.275$\pm$0.015
& 0.401$\pm$0.189 & 0.854$\pm$0.042 & 0.528$\pm$0.087 \\

\textsc{\textbf{CITE-VAE}} (ours)
& \cmark & \cmark & \cmark
& 0.086$\pm$0.054 & $>0.95$ & 0.033$\pm$0.015
& 0.691$\pm$0.007 & 0.931$\pm$0.020 & 0.617$\pm$0.061 \\
\bottomrule
\end{tabular}
\end{adjustbox}
\end{table*}

\begin{figure*}[h] \centering \includegraphics[width=0.85\textwidth]{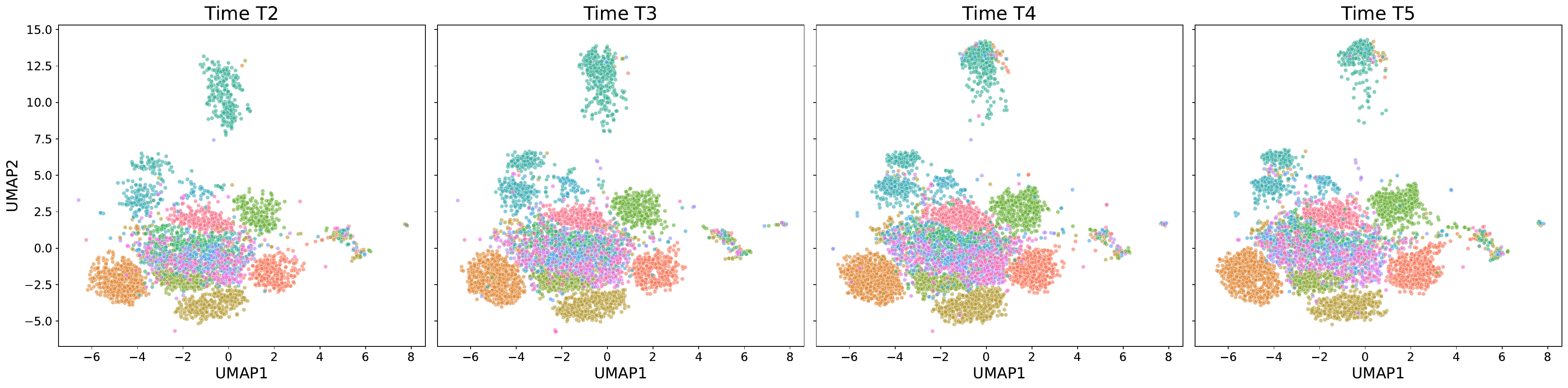} \caption{\textbf{Latent-space organization under temporal perturbations.}
UMAP of $\mathbf{z}_\nu$ for the 22 training genes across four time points, which shows perturbation-level separability and consistent temporal displacement across genes.}
\label{fig:umap}
\end{figure*}

\textit{Evaluation.} We evaluate pseudobulk-level prediction accuracy both genome-wide (all genes) and on two fold-specific differentially expressed (DE) gene subsets (top-100 DE genes per fold) across 3 seeds, with reporting mean and standard deviation.
Reporting genome-wide metrics assesses overall transcriptome fidelity, whereas DE-restricted metrics focus the evaluation on perturbation-responsive signals that are most relevant for interpreting intervention effects. To prevent information leakage, DE genes are identified using training interventions only within each cross-validation fold and then kept fixed when evaluating the held-out perturbation. We include all genes and DE-genes to probe robustness to the stringency of gene selection and to separate performance on moderately versus strongly perturbed programs.

\textit{Quantitative Results.}
The results are summarized in \Cref{tab:hipsc_full_eval}.
The evaluation is intentionally multi-faceted: genome-wide metrics assess overall transcriptome fidelity, whereas DE-focused metrics probe whether models capture perturbation-responsive programs. Across all metrics, \textsc{CITE-VAE} exhibits consistently strong and non-degenerate performance.
Compared with classical linear baselines such as ElasticNet~\citep{zou2005regularization}, which can achieve strong genome-wide fits but lack explicit temporal or causal structure, \textsc{CITE-VAE} provides substantially improved performance on DE-focused criteria that are central to perturbation effect prediction. When compared to methods grounded in explicit causal formulations, such as \textsc{SAMS-VAE} and \textsc{Discrepancy-VAE}, \textsc{CITE-VAE} achieves systematically better results on DE-sensitive metrics while remaining competitive on genome-wide measures, highlighting the benefit of jointly modeling causal structure and temporal dynamics. Moreover, \textsc{CITE-VAE} remains competitive even relative to large foundation models such as \textsc{UCE} and \textsc{scGPT}.
Although these models leverage substantially larger training data and architectures—rendering the comparison inherently conservative, \textsc{CITE-VAE} attains comparable genome-wide accuracy and stronger DE-focused performance.
Overall, these results indicate that principled causal and dynamical modeling can effectively complement, and in some cases rival, scale-driven approaches for temporal perturbation effect prediction.

\textit{Qualitative Results.} We visualize the intervention-responsive latent states $\mathbf{z}_\nu$ for the 22 training perturbations across four time points (\Cref{fig:umap}). The embeddings exhibit distinct cluster-wise separation, confirming that $\mathbf{z}_\nu$ retains the high-dimensional heterogeneity of gene-specific responses. Moreover, these clusters demonstrate smooth and directionally consistent displacement over time, suggesting that \textsc{CITE-VAE} has identified a universal, system-level dynamical operator rather than overfitting to snapshot-specific features.  This structured latent organization provides a robust geometric foundation for causal extrapolation to the held-out perturbation. 

\section{Conclusion and Limitations}
\label{sec:Conclusion}
We studied the problem of learning perturbation-driven cellular dynamics from single-cell snapshots. Motivated by the fact that perturbation effects are inherently both \emph{latent} and \emph{dynamical}, we proposed a latent dynamical causal generative model that captures unobserved cellular programs and their temporal evolution under interventions. We provided an identifiability analysis showing that, under suitable conditions, latent causal variables are recoverable up to a trivial equivalence class when likelihood-based learning is combined with invariant alignment across environments. Guided by this analysis, we developed \textsc{CITE-VAE}, which translates population-level identifiability conditions into concrete modeling and optimization choices. Experiments on benchmark datasets and real-world single-cell perturbation studies demonstrate improved generalization to unseen perturbations. 

A limitation of this work, shared with most studies on causal representation learning, lies in the assumptions required for the theoretical identifiability analysis. In general, recovering latent causal variables without assumptions is impossible. Consequently, as in prior work, we impose a set of conditions, e.g., the proposed latent dynamical causal model defined in Eqs.~\ref{eq:dgp_noise_nu}-\ref{eq:dgp_x},  as well as assumptions~\ref{asm:ts_rank}-\ref{asm:ts_isolation} to achieve  identifiability. This limitation is inherent to causal representation learning and is not unique to this work. In addition, our empirical results on real-world 
data indicate that the proposed method performs robustly in practice,
suggesting that the theoretical insights remain informative even when the
assumptions are only approximately satisfied. Finally, extending our framework to more complex settings is an important direction for future work.

\section*{Acknowledgment}
The work is partly supported by the Responsible AI Research Centre and ARC Discovery Project DP240103278. 

\section*{Limitations and Ethical Considerations}
One limitation of this work lies in the assumptions imposed on the underlying generative process and those required for identifiability. These assumptions, while standard in causal representation learning, might not hold in some real cases and are discussed in detail in~\Cref{sec:Conclusion}. We confirm that this study complies with the ethical standards of SIGKDD, with no involvement of private or sensitive information.

\section*{GenAI Disclosure}
The authors used generative AI tools solely for English-language editing to improve clarity, grammar, and style. These tools were not used to develop the scientific ideas, design the methodology, conduct experiments, analyze results, or write technical content. All experimental results, figures, and tables were produced by the authors without the use of generative AI.


\clearpage

\bibliographystyle{ACM-Reference-Format}
\bibliography{kdd2026/reference}

@article{von2025learning,
  title={Learning Explicit Single-Cell Dynamics Using ODE Representations},
  author={von Bassewitz, Jan-Philipp and Pervez, Adeel and Fumero, Marco and Robinson, Matthew and Karaletsos, Theofanis and Locatello, Francesco},
  journal={arXiv preprint arXiv:2510.02903},
  year={2025}
}

@article{tejada2025causal,
  title={Causal machine learning for single-cell genomics},
  author={Tejada-Lapuerta, Alejandro and Bertin, Paul and Bauer, Stefan and Aliee, Hananeh and Bengio, Yoshua and Theis, Fabian J},
  journal={Nature Genetics},
  pages={1--12},
  year={2025},
  publisher={Nature Publishing Group US New York}
}

@article{mao2025halo,
  title={HALO: hierarchical causal modeling for single cell multi-omics data},
  author={Mao, Haiyi and Jia, Minxue and Di, Marissa and Valenzi, Eleanor and Cai, Xiaoyu Tracy and Lafyatis, Robert and Zhang, Kun and Benos, Panayiotis V},
  journal={Nature Communications},
  volume={16},
  number={1},
  pages={8892},
  year={2025},
  publisher={Nature Publishing Group UK London}
}

@article{wu2025unveiling,
  title={Unveiling causal regulatory mechanisms through cell-state parallax},
  author={Wu, Alexander P and Singh, Rohit and Walsh, Christopher A and Berger, Bonnie},
  journal={Nature Communications},
  volume={16},
  number={1},
  pages={8096},
  year={2025},
  publisher={Nature Publishing Group UK London}
}

@article{singh2024causal,
  title={Causal gene regulatory analysis with RNA velocity reveals an interplay between slow and fast transcription factors},
  author={Singh, Rohit and Wu, Alexander P and Mudide, Anish and Berger, Bonnie},
  journal={Cell systems},
  volume={15},
  number={5},
  pages={462--474},
  year={2024},
  publisher={Elsevier}
}

@article{borsari2025chronode,
  title={The chronODE framework for modelling multi-omic time series with ordinary differential equations and machine learning},
  author={Borsari, Beatrice and Frank, Mor and Wattenberg, Eve S and Xu, Ke and Liu, Susanna X and Yu, Xuezhu and Gerstein, Mark},
  journal={Nature Communications},
  volume={16},
  number={1},
  pages={7021},
  year={2025},
  publisher={Nature Publishing Group UK London}
}

@article{moscardo2025multi,
  title={Multi-omic network inference from time-series data},
  author={Moscard{\'o} Garc{\'\i}a, Mar{\'\i}a and Aalto, Atte and Montanari, Arthur N and Skupin, Alexander and Gon{\c{c}}alves, Jorge},
  journal={npj Systems Biology and Applications},
  volume={11},
  number={1},
  pages={114},
  year={2025},
  publisher={Nature Publishing Group UK London}
}

@article{norman2019exploring,
  title={Exploring genetic interaction manifolds constructed from rich single-cell phenotypes},
  author={Norman, Thomas M and Horlbeck, Max A and Replogle, Joseph M and Ge, Alex Y and Xu, Albert and Jost, Marco and Gilbert, Luke A and Weissman, Jonathan S},
  journal={Science},
  volume={365},
  number={6455},
  pages={786--793},
  year={2019},
  publisher={American Association for the Advancement of Science}
}

@article{zhang2023identifiability,
  title={Identifiability guarantees for causal disentanglement from soft interventions},
  author={Zhang, Jiaqi and Greenewald, Kristjan and Squires, Chandler and Srivastava, Akash and Shanmugam, Karthikeyan and Uhler, Caroline},
  journal={Advances in Neural Information Processing Systems},
  volume={36},
  pages={50254--50292},
  year={2023}
}

@article{von2021self,
  title={Self-supervised learning with data augmentations provably isolates content from style},
  author={Von K{\"u}gelgen, Julius and Sharma, Yash and Gresele, Luigi and Brendel, Wieland and Sch{\"o}lkopf, Bernhard and Besserve, Michel and Locatello, Francesco},
  journal={Advances in neural information processing systems},
  volume={34},
  pages={16451--16467},
  year={2021}
}

@article{hyvarinen1999nonlinear,
  title={Nonlinear independent component analysis: Existence and uniqueness results},
  author={Hyv{\"a}rinen, Aapo and Pajunen, Petteri},
  journal={Neural networks},
  volume={12},
  number={3},
  pages={429--439},
  year={1999},
  publisher={Elsevier}
}

@incollection{hyvarinen2001independent,
  title={Independent component analysis},
  author={Hyv{\"a}rinen, Aapo and Hurri, Jarmo and Hoyer, Patrik O},
  booktitle={Natural Image Statistics: A Probabilistic Approach to Early Computational Vision},
  pages={151--175},
  year={2001},
  publisher={Springer}
}

@article{liu2022identifying,
  title={Identifying weight-variant latent causal models},
  author={Liu, Yuhang and Zhang, Zhen and Gong, Dong and Gong, Mingming and Huang, Biwei and van den Hengel, Anton and Zhang, Kun and Shi, Javen Qinfeng},
  journal={Journal of Machine Learning Research},
  volume={27},
  number={4},
  pages={1--49},
  year={2026}
}

@inproceedings{
liu2026i,
title={I Predict Therefore I Am: Is Next Token Prediction Enough to Learn Human-Interpretable Concepts from Data?},
author={Yuhang Liu and Dong Gong and Yichao Cai and Erdun Gao and Zhen Zhang and Biwei Huang and Mingming Gong and Anton van den Hengel and Javen Qinfeng Shi},
booktitle={The Fourteenth International Conference on Learning Representations},
year={2026}
}

@inproceedings{
liu2026beyond,
title={Beyond {DAG}s: A Latent Partial Causal Model for Multimodal Learning},
author={Yuhang Liu and Zhen Zhang and Dong Gong and Erdun Gao and Biwei Huang and Mingming Gong and Anton van den Hengel and Kun Zhang and Javen Qinfeng Shi},
booktitle={The Fourteenth International Conference on Learning Representations},
year={2026},
}

@article{
liu2025latent,
title={Latent Covariate Shift: Unlocking Partial Identifiability for Multi-Source Domain Adaptation},
author={Yuhang Liu and Zhen Zhang and Dong Gong and Mingming Gong and Biwei Huang and Anton van den Hengel and Kun Zhang and Javen Qinfeng Shi},
journal={Transactions on Machine Learning Research},
issn={2835-8856},
year={2025},
}

@article{liu2024towards,
  title={Towards Identifiable Latent Additive Noise Models},
  author={Liu, Yuhang and Zhang, Zhen and Gong, Dong and Gao, Erdun and Huang, Biwei and Gong, Mingming and Hengel, Anton van den and Zhang, Kun and Shi, Javen Qinfeng},
  journal={arXiv preprint arXiv:2403.15711},
  year={2024}
}

@article{hyvarinen2016unsupervised,
  title={Unsupervised feature extraction by time-contrastive learning and nonlinear ica},
  author={Hyvarinen, Aapo and Morioka, Hiroshi},
  journal={Advances in neural information processing systems},
  volume={29},
  year={2016}
}

@inproceedings{hyvarinen2017nonlinear,
  title={Nonlinear ICA of temporally dependent stationary sources},
  author={Hyvarinen, Aapo and Morioka, Hiroshi},
  booktitle={Artificial intelligence and statistics},
  pages={460--469},
  year={2017},
  organization={PMLR}
}

@inproceedings{
liu2023identifiable,
title={Identifiable Latent Polynomial Causal Models through the Lens of Change},
author={Yuhang Liu and Zhen Zhang and Dong Gong and Mingming Gong and Biwei Huang and Anton van den Hengel and Kun Zhang and Javen Qinfeng Shi},
booktitle={The Twelfth International Conference on Learning Representations},
year={2024}
}

@article{sorrenson2020disentanglement,
  title={Disentanglement by nonlinear ica with general incompressible-flow networks (gin)},
  author={Sorrenson, Peter and Rother, Carsten and K{\"o}the, Ullrich},
  journal={arXiv preprint arXiv:2001.04872},
  year={2020}
}

@inproceedings{khemakhem2020variational,
  title={Variational autoencoders and nonlinear ica: A unifying framework},
  author={Khemakhem, Ilyes and Kingma, Diederik and Monti, Ricardo and Hyvarinen, Aapo},
  booktitle={International conference on artificial intelligence and statistics},
  pages={2207--2217},
  year={2020},
  organization={PMLR}
}

@misc{kingma2013auto,
  title={Auto-encoding variational bayes},
  author={Kingma, Diederik P and Welling, Max and others},
  year={2013},
  publisher={Banff, Canada}
}

@inproceedings{rezende2014stochastic,
  title={Stochastic backpropagation and approximate inference in deep generative models},
  author={Rezende, Danilo Jimenez and Mohamed, Shakir and Wierstra, Daan},
  booktitle={International conference on machine learning},
  pages={1278--1286},
  year={2014},
  organization={PMLR}
}

@inproceedings{lippe2022citris,
  title={Citris: Causal identifiability from temporal intervened sequences},
  author={Lippe, Phillip and Magliacane, Sara and L{\"o}we, Sindy and Asano, Yuki M and Cohen, Taco and Gavves, Stratis},
  booktitle={International Conference on Machine Learning},
  pages={13557--13603},
  year={2022},
  organization={PMLR}
}

@article{lotfollahi2023predicting,
  title={Predicting cellular responses to complex perturbations in high-throughput screens},
  author={Lotfollahi, Mohammad and Klimovskaia Susmelj, Anna and De Donno, Carlo and Hetzel, Leon and Ji, Yuge and Ibarra, Ignacio L and Srivatsan, Sanjay R and Naghipourfar, Mohsen and Daza, Riza M and Martin, Beth and others},
  journal={Molecular systems biology},
  volume={19},
  number={6},
  pages={e11517},
  year={2023}
}

@article{de2025interpretable,
  title={Interpretable Causal Representation Learning for Biological Data in the Pathway Space},
  author={de la Fuente, Jesus and Lehmann, Robert and Ruiz-Arenas, Carlos and Voges, Jan and Marin-Go{\~n}i, Irene and Martinez-de-Morentin, Xabier and Gomez-Cabrero, David and Ochoa, Idoia and Tegner, Jesper and Lagani, Vincenzo and others},
  journal={arXiv preprint arXiv:2506.12439},
  year={2025}
}

@inproceedings{ahuja2023interventional,
  title={Interventional causal representation learning},
  author={Ahuja, Kartik and Mahajan, Divyat and Wang, Yixin and Bengio, Yoshua},
  booktitle={International conference on machine learning},
  pages={372--407},
  year={2023},
  organization={PMLR}
}

@inproceedings{
lopez2022learning,
title={Learning Causal Representations of Single Cells via Sparse Mechanism Shift Modeling},
author={Romain Lopez and Natasa Tagasovska and Stephen Ra and Kyunghyun Cho and Jonathan Pritchard and Aviv Regev},
booktitle={NeurIPS 2022 Workshop on Causality for Real-world Impact},
year={2022},
url={https://openreview.net/forum?id=gdTXCy7fZf7}
}

@inproceedings{
lachapelle2022disentanglement,
title={Disentanglement via Mechanism Sparsity Regularization: A New Principle for Nonlinear {ICA}},
author={Sebastien Lachapelle and Pau Rodriguez and Yash Sharma and Katie E Everett and R{\'e}mi LE PRIOL and Alexandre Lacoste and Simon Lacoste-Julien},
booktitle={First Conference on Causal Learning and Reasoning},
year={2022},
url={https://openreview.net/forum?id=dHsFFekd_-o}
}

@article{scholkopf2021toward,
  title={Toward causal representation learning},
  author={Sch{\"o}lkopf, Bernhard and Locatello, Francesco and Bauer, Stefan and Ke, Nan Rosemary and Kalchbrenner, Nal and Goyal, Anirudh and Bengio, Yoshua},
  journal={Proceedings of the IEEE},
  volume={109},
  number={5},
  pages={612--634},
  year={2021},
  publisher={IEEE}
}

@article{replogle2020combinatorial,
  title={Combinatorial single-cell CRISPR screens by direct guide RNA capture and targeted sequencing},
  author={Replogle, Joseph M and Norman, Thomas M and Xu, Albert and Hussmann, Jeffrey A and Chen, Jin and Cogan, J Zachery and Meer, Elliott J and Terry, Jessica M and Riordan, Daniel P and Srinivas, Niranjan and others},
  journal={Nature biotechnology},
  volume={38},
  number={8},
  pages={954--961},
  year={2020},
  publisher={Nature Publishing Group US New York}
}

@article{von2023nonparametric,
  title={Nonparametric identifiability of causal representations from unknown interventions},
  author={von K{\"u}gelgen, Julius and Besserve, Michel and Wendong, Liang and Gresele, Luigi and Keki{\'c}, Armin and Bareinboim, Elias and Blei, David and Sch{\"o}lkopf, Bernhard},
  journal={Advances in Neural Information Processing Systems},
  volume={36},
  pages={48603--48638},
  year={2023}
}

@article{gao2025domain,
  title={Domain generalization via content factors isolation: a two-level latent variable modeling approach},
  author={Gao, Erdun and Bondell, Howard and Huang, Shaoli and Gong, Mingming},
  journal={Machine Learning},
  volume={114},
  number={4},
  pages={1--33},
  year={2025},
  publisher={Springer}
}

@inproceedings{zimmermann2021contrastive,
  title={Contrastive learning inverts the data generating process},
  author={Zimmermann, Roland S and Sharma, Yash and Schneider, Steffen and Bethge, Matthias and Brendel, Wieland},
  booktitle={International conference on machine learning},
  pages={12979--12990},
  year={2021},
  organization={PMLR}
}

@article{ishikawa2023renge,
  title={RENGE infers gene regulatory networks using time-series single-cell RNA-seq data with CRISPR perturbations},
  author={Ishikawa, Masato and Sugino, Seiichi and Masuda, Yoshie and Tarumoto, Yusuke and Seto, Yusuke and Taniyama, Nobuko and Wagai, Fumi and Yamauchi, Yuhei and Kojima, Yasuhiro and Kiryu, Hisanori and others},
  journal={Communications Biology},
  volume={6},
  number={1},
  pages={1290},
  year={2023},
  publisher={Nature Publishing Group UK London}
}

@article{cui2024scgpt,
  title={scGPT: toward building a foundation model for single-cell multi-omics using generative AI},
  author={Cui, Haotian and Wang, Chloe and Maan, Hassaan and Pang, Kuan and Luo, Fengning and Duan, Nan and Wang, Bo},
  journal={Nature methods},
  volume={21},
  number={8},
  pages={1470--1480},
  year={2024},
  publisher={Nature Publishing Group US New York}
}

@article{rosen2023universal,
  title={Universal cell embeddings: A foundation model for cell biology},
  author={Rosen, Yanay and Roohani, Yusuf and Agarwal, Ayush and Samotor{\v{c}}an, Leon and Tabula Sapiens Consortium and Quake, Stephen R and Leskovec, Jure},
  journal={bioRxiv},
  pages={2023--11},
  year={2023},
  publisher={Cold Spring Harbor Laboratory}
}

@article{zou2005regularization,
  title={Regularization and variable selection via the elastic net},
  author={Zou, Hui and Hastie, Trevor},
  journal={Journal of the Royal Statistical Society Series B: Statistical Methodology},
  volume={67},
  number={2},
  pages={301--320},
  year={2005},
  publisher={Oxford University Press}
}

@article{lippe2022causal,
  title={Causal representation learning for instantaneous and temporal effects in interactive systems},
  author={Lippe, Phillip and Magliacane, Sara and L{\"o}we, Sindy and Asano, Yuki M and Cohen, Taco and Gavves, Efstratios},
  journal={arXiv preprint arXiv:2206.06169},
  year={2022}
}

@article{ashburner2000gene,
  title={Gene ontology: tool for the unification of biology},
  author={Ashburner, Michael and Ball, Catherine A and Blake, Judith A and Botstein, David and Butler, Heather and Cherry, J Michael and Davis, Allan P and Dolinski, Kara and Dwight, Selina S and Eppig, Janan T and others},
  journal={Nature genetics},
  volume={25},
  number={1},
  pages={25--29},
  year={2000},
  publisher={Nature Publishing Group}
}

@article{bereket2023modelling,
  title={Modelling cellular perturbations with the sparse additive mechanism shift variational autoencoder},
  author={Bereket, Michael and Karaletsos, Theofanis},
  journal={Advances in Neural Information Processing Systems},
  volume={36},
  pages={1--12},
  year={2023}
}

@article{zhang2024scnode,
  title={scNODE: generative model for temporal single cell transcriptomic data prediction},
  author={Zhang, Jiaqi and Larschan, Erica and Bigness, Jeremy and Singh, Ritambhara},
  journal={Bioinformatics},
  volume={40},
  number={Supplement\_2},
  pages={ii146--ii154},
  year={2024},
  publisher={Oxford University Press}
}

@article{balmas2025single,
  title={Single cell transcriptional perturbome in pluripotent stem cell models},
  author={Balmas, Elisa and Ratto, Maria L and Snijders, Kirsten E and Becca, Silvia and Liaci, Carla and Ricca, Irene and Merlo, Giorgio R and Calogero, Raffaele A and Alessandr{\`\i}, Luca and Mendjan, Sasha and others},
  journal={Molecular Systems Biology},
  volume={22},
  number={2},
  pages={179},
  year={2025}
}

@article{bunne2023learning,
  title={Learning single-cell perturbation responses using neural optimal transport},
  author={Bunne, Charlotte and Stark, Stefan G and Gut, Gabriele and Del Castillo, Jacobo Sarabia and Levesque, Mitch and Lehmann, Kjong-Van and Pelkmans, Lucas and Krause, Andreas and R{\"a}tsch, Gunnar},
  journal={Nature methods},
  volume={20},
  number={11},
  pages={1759--1768},
  year={2023},
  publisher={Nature Publishing Group US New York}
}

@article{pandarinath2018inferring,
  title={Inferring single-trial neural population dynamics using sequential auto-encoders},
  author={Pandarinath, Chethan and O’Shea, Daniel J and Collins, Jasmine and Jozefowicz, Rafal and Stavisky, Sergey D and Kao, Jonathan C and Trautmann, Eric M and Kaufman, Matthew T and Ryu, Stephen I and Hochberg, Leigh R and others},
  journal={Nature methods},
  volume={15},
  number={10},
  pages={805--815},
  year={2018},
  publisher={Nature Publishing Group US New York}
}

@inproceedings{yang2021causalvae,
  title={Causalvae: Disentangled representation learning via neural structural causal models},
  author={Yang, Mengyue and Liu, Furui and Chen, Zhitang and Shen, Xinwei and Hao, Jianye and Wang, Jun},
  booktitle={Proceedings of the IEEE/CVF conference on computer vision and pattern recognition},
  pages={9593--9602},
  year={2021}
}

@article{jiang2026makes,
  title={What Makes a Representation Good for Single-Cell Perturbation Prediction?},
  author={Jiang, Wenkang and Liu, Yuhang and Cai, Yichao and Gao, Erdun and Dong, Jiayi and Abbasnejad, Ehsan and Yao, Lina and Shi, Javen Qinfeng},
  journal={arXiv preprint arXiv:2605.19343},
  year={2026}
}

\appendix

\appendix

\section{Proof of Theorem~\ref{thm:ts_identifiability}}
\label{app:proof_identifiability}
\begin{proof}
We proceed in three steps: (I) reduce transition matching to innovation matching,
(II) identify the innovations $(\mathbf n_\iota^{t},\mathbf n_\nu^{t})$,
and (III) lift innovation identifiability to state identifiability.

\textbf{Step I: From $(\mathbf x^{t-1},\mathbf x^{t})$ to $(\mathbf z^t,\mathbf n^{t})$ via a triangular Jacobian.}
By invertibility of $\mathbf{g}$, define $\mathbf z^t=\mathbf{g}^{-1}(\mathbf x^t)$.
Define the innovation reconstruction map
\begin{align}
\mathbf n^{t}
&:=
\mathbf z^{t}-\mathbf F(\mathbf z^{t-1},\mathbf u),\\
\mathbf F(\mathbf z^{t-1},\mathbf u)&:=
\big(\,\mathbf{f}_\iota(\mathbf z_\iota^{t-1}),\ f_{\nu,1}(\mathbf z_{\mathrm{pa}(1)}^{t-1},\mathbf u),\ldots,f_{\nu,d_\nu}(\mathbf z_{\mathrm{pa}(d_\nu)}^{t-1},\mathbf u)\,\big).
\end{align}
Then
\begin{equation}
\mathbf n^{t}
=
\mathbf{g}^{-1}(\mathbf x^{t})-\mathbf F\big(\mathbf{g}^{-1}(\mathbf x^{t-1}),\mathbf u\big).
\label{eq:ts_n_from_x}
\end{equation}
Consider the change of variables
\begin{equation}
(\mathbf x^{t-1},\mathbf x^{t})
\ \mapsto\
(\mathbf z^{t-1},\mathbf n^{t}).
\end{equation}
Its Jacobian matrix is block lower triangular:
\begin{equation}
\mathbf{J} =
\begin{pmatrix}
\mathbf{J}_{\mathbf{g}^{-1}}(\mathbf x^{t-1}) & \mathbf 0\\
\ast & \mathbf{J}_{\mathbf{g}^{-1}}(\mathbf x^{t})
\end{pmatrix},
\end{equation}
hence
\begin{equation}
\big|\det \mathbf{J}\big|
=
\big|\det \mathbf{J}_{\mathbf{g}^{-1}}(\mathbf x^{t-1})\big|
\cdot
\big|\det \mathbf{J}_{\mathbf{g}^{-1}}(\mathbf x^{t})\big|.
\label{eq:ts_triangular_det}
\end{equation}
According to Eqs.~\ref{eq:dgp_noise_nu}-\ref{eq:dgp_znu}, $\mathbf n^{t}\perp \mathbf z^{t-1} \mid \mathbf u$ and
$p(\mathbf n^{t}\mid \mathbf z^{t-1},\mathbf u)=p(\mathbf n^{t}\mid \mathbf u)
=p(\mathbf n_\iota^{t})\,p(\mathbf n_\nu^{t}\mid \mathbf u)$.
Therefore the transition density factors as
\begin{align}
p_\theta(\mathbf x^{t}\mid \mathbf x^{t-1},\mathbf u)
&=
p_\theta(\mathbf n^{t}\mid \mathbf u)\cdot \big|\det \mathbf{J}_{\mathbf{g}^{-1}}(\mathbf x^{t})\big|
\label{eq:ts_cov_transition}
\end{align}
where the term $\big|\det \mathbf{J}_{\mathbf{g}^{-1}}(\mathbf x^{t})\big|$ is independent of $\mathbf u$.
Applying the same identity to $\hat\theta$ and using the transition matching assumption
Eq.~\eqref{eq:ts_match_transition}, we conclude that for all environments
\begin{equation}
\log \frac{p_\theta(\mathbf n^{t}\mid \mathbf u)}{p_\theta(\mathbf n^{t}\mid \mathbf u_0)}
=
\log \frac{p_{\hat\theta}(\hat{\mathbf n}^{t}\mid \mathbf u)}{p_{\hat\theta}(\hat{\mathbf n}^{t}\mid \mathbf u_0)}
\quad \text{a.s.}
\label{eq:ts_ratio_n}
\end{equation}
Moreover, since $p(\mathbf n^{t}\mid \mathbf u)=p(\mathbf n_\iota^{t})p(\mathbf n_\nu^{t}\mid \mathbf u)$
and $p(\mathbf n_\iota^{t})$ is $\mathbf u$-invariant, differencing across environments removes the $\iota$-part,
leaving an identity involving only $\mathbf n_\nu^{t}$.

\textbf{Step II: Identifiability of $\mathbf n^{t}$.}
\paragraph{(II-a) Component-wise identifiability of $\mathbf n_\nu^{t}$.}
Under diagonal Gaussian conditionals, the log-ratio in Eq.~\eqref{eq:ts_ratio_n} is linear in the sufficient statistics
$\mathbf T_\nu(\mathbf n_\nu^{t})=(\mathbf n_\nu^{t},(\mathbf n_\nu^{t})^{\odot 2})$.
Stacking across $2d_\nu$ environments and using Assumption~\ref{asm:ts_rank}, the standard nonlinear-ICA
rank argument \citep{khemakhem2020variational,sorrenson2020disentanglement} yields an affine relation
\begin{equation}
\mathbf T_\nu(\mathbf n_\nu^{t})
=
\mathbf A_\nu\,\hat{\mathbf T}_\nu(\hat{\mathbf n}_\nu^{t})+\mathbf c_\nu,
\label{eq:ts_affine_T}
\end{equation}
with $\mathbf A_\nu$ full rank. Since the coordinates of $\mathbf n_\nu^{t}$ are mutually independent Gaussians,
Eq.~\eqref{eq:ts_affine_T} forces each $n_{\nu,j}^{t}$ to depend on exactly one coordinate of
$\hat{\mathbf n}_\nu^{t}$ up to an affine transform~\citep{sorrenson2020disentanglement}. Hence there exist a permutation $\pi$,
nonzero scalars $\{a_j\}$ and constants $\{b_j\}$ such that
\begin{equation}
n_{\nu,j}^{t}
=
a_j\,\hat n_{\nu,\pi(j)}^{t}+b_j,\qquad j=1,\ldots,d_\nu.
\label{eq:ts_id_n_nu}
\end{equation}

\paragraph{(II-b) Block identifiability of $\mathbf n_\iota^{t}$ via temporal alignment.}
Recall the alignment objective in Assumption~\ref{asm:ts_align},
\begin{equation}
\mathcal L_{\mathrm{align}}
=
\mathbb E\Big[
\big\|
\mathbf{\hat f}_\iota(\mathbf x^{t,(\mathbf u)})
-
\mathbf{\hat f}_\iota(\mathbf x^{t,(\mathbf u_0)})
\big\|_2^2
\Big],
\end{equation}
defined under the coupled cross-environment construction where
$\mathbf x^{t,(\mathbf u)}$ and $\mathbf x^{t,(\mathbf u_0)}$ share the same invariant state
$\mathbf z_\iota^t$ but differ only through independent responsive randomness
(e.g., independent draws of $\mathbf n_\nu^t$ under $\mathbf u$ and $\mathbf u_0$).
Since $\mathcal L_{\mathrm{align}}=0$ at the population global minimum, we have
\begin{equation}
\mathbf{\hat f}_\iota(\mathbf x^{t,(\mathbf u)})=\mathbf{\hat f}_\iota(\mathbf x^{t,(\mathbf u_0)})
\quad \text{a.s.}
\label{eq:ts_align_zero_implies_equal}
\end{equation}
for any environments $\mathbf u,\mathbf u_0$.
Fix any $t$ and any value of the invariant state $\mathbf z_\iota^t$.
Under the above coupling, Eq.~\eqref{eq:ts_align_zero_implies_equal} implies that
$\mathbf{\hat f}_\iota(\mathbf x^{t,(\mathbf u)})$ is almost surely constant with respect to the responsive
randomness (since varying the responsive randomness does not change $\mathbf z_\iota^t$).
Therefore, there exists a measurable function $h_\iota$ such that
\begin{equation}
\mathbf{\hat f}_\iota(\mathbf x^{t})=\mathbf{h}_\iota(\mathbf z_\iota^t)
\qquad\text{a.s.}
\label{eq:ts_h_iota_def}
\end{equation}

Using the invariant update $\mathbf z_\iota^{t}=\mathbf{f}_\iota(\mathbf z_\iota^{t-1})+\mathbf n_\iota^{t}$,
define the induced innovation representation
\begin{equation}
\hat{\mathbf n}_\iota^{t}
:=
\mathbf{h}_\iota(\mathbf z_\iota^{t})-\mathbf{h}_\iota(f_\iota(\mathbf z_\iota^{t-1}))
=
\mathbf{h}_\iota\!\big(f_\iota(\mathbf z_\iota^{t-1})+\mathbf n_\iota^{t}\big)-\mathbf{h}_\iota\!\big(f_\iota(\mathbf z_\iota^{t-1})\big).
\label{eq:ts_hat_n_iota_func}
\end{equation}
For each fixed $\mathbf z_\iota^{t-1}$, the mapping $\mathbf n_\iota^{t}\mapsto \hat{\mathbf n}_\iota^{t}$
is a smooth bijection on $\mathbb R^{d_\iota}$, since $\mathbf{h}_\iota$ is smooth (as $\mathbf{\hat f}_\iota$ and $\mathbf{g}$ are smooth
and invertible) and $\mathbf n_\iota^{t}$ is non-degenerate.
Moreover, $\mathbf n_\iota^{t}$ is Gaussian and independent of $\mathbf z_\iota^{t-1}$, hence
$\hat{\mathbf n}_\iota^{t}$ is also Gaussian (with possibly different mean and covariance).
By the standard Gaussian-to-Gaussian diffeomorphism argument used in the static case,
any smooth bijection mapping a non-degenerate Gaussian to a Gaussian must be affine.
Consequently, there exist a nonsingular matrix $\mathbf A_n$ and a vector $\mathbf b_n$ such that
\begin{equation}
\hat{\mathbf n}_\iota^{t} = \mathbf A_n\,\mathbf n_\iota^{t}+\mathbf b_n,
\label{eq:ts_id_n_iota}
\end{equation}
which establishes block-wise linear identifiability.

Combining Eqs.~\eqref{eq:ts_id_n_nu} and \eqref{eq:ts_id_n_iota}, we can write the global innovation relation as
\begin{equation}
\hat{\mathbf n}^{t}
=
\mathbf A_{\mathrm{glob}}\mathbf n^{t}+\mathbf b,
\qquad
\mathbf A_{\mathrm{glob}}=
\begin{pmatrix}
\mathbf A_n & \mathbf 0\\
\mathbf 0 & \mathbf P_n
\end{pmatrix},
\label{eq:ts_global_innov}
\end{equation}
where $\mathbf P_n$ is a scaling-permutation matrix.

\textbf{Step III: Identifiability of $\mathbf z^t$.}
Since $\mathbf x^t=\mathbf{g}(\mathbf z^t)$ and $\hat{\mathbf z}^t=\mathbf{\hat g}^{-1}(\mathbf x^t)$,
the composition
\begin{equation}
\mathbf h := \mathbf{g}^{-1}\circ \mathbf{\hat g}
\end{equation}
satisfies
\begin{equation}
\mathbf z^t=\mathbf h(\hat{\mathbf z}^t)
\qquad\text{for all } t.
\label{eq:ts_h_def}
\end{equation}
Because both $\mathbf{g}$ and $\mathbf{\hat g}$ are time-invariant and independent of the environment,
the map $\mathbf h$ is also independent of $t$ and $\mathbf u$.
We now characterize the structure of $\mathbf h$.

\paragraph{Step III-a: $\mathbf h$ is block-affine.}
From Step~II, the innovation vectors satisfy
\begin{equation}
\hat{\mathbf n}^{t}
=
\mathbf A_{\mathrm{glob}}\,\mathbf n^{t}+\mathbf b,
\qquad
\mathbf A_{\mathrm{glob}}
=
\begin{pmatrix}
\mathbf A_n & \mathbf 0\\
\mathbf 0 & \mathbf P_n
\end{pmatrix},
\label{eq:ts_global_noise_affine}
\end{equation}
where $\mathbf A_n$ is nonsingular and $\mathbf P_n$ is a scaling-permutation matrix.
Recalling the additive updates
\[
\mathbf z^{t}=\mathbf F(\mathbf z^{t-1},\mathbf u)+\mathbf n^{t},
\qquad
\hat{\mathbf z}^{t}=\hat{\mathbf F}(\hat{\mathbf z}^{t-1},\mathbf u)+\hat{\mathbf n}^{t},
\]
and substituting Eq.~\eqref{eq:ts_h_def} into both equations, we obtain
\begin{equation}
\mathbf h(\hat{\mathbf z}^{t})
-
\mathbf F(\mathbf h(\hat{\mathbf z}^{t-1}),\mathbf u)
=
\mathbf A_{\mathrm{glob}}
\big(
\hat{\mathbf z}^{t}
-
\hat{\mathbf F}(\hat{\mathbf z}^{t-1},\mathbf u)
\big)
+
\mathbf b.
\label{eq:ts_h_functional}
\end{equation}
Since Eq.~\eqref{eq:ts_h_functional} holds for all $\hat{\mathbf z}^{t-1},\hat{\mathbf z}^{t}$ and all environments,
and the right-hand side is affine in $\hat{\mathbf z}^{t}$, it follows that $\mathbf h$ itself must be affine.
Thus there exist a constant matrix $\mathbf M$ and vector $\mathbf c$ such that
\begin{equation}
\mathbf h(\hat{\mathbf z})=\mathbf M\hat{\mathbf z}+\mathbf c.
\label{eq:ts_h_affine}
\end{equation}
Moreover, the block-diagonal structure of $\mathbf A_{\mathrm{glob}}$ in
Eq.~\eqref{eq:ts_global_noise_affine} implies that $\mathbf M$ admits a block decomposition
\begin{equation}
\mathbf M
=
\begin{pmatrix}
\mathbf M_{\iota\iota} & \mathbf M_{\iota\nu}\\
\mathbf M_{\nu\iota} & \mathbf M_{\nu\nu}
\end{pmatrix},
\label{eq:ts_M_blocks}
\end{equation}
where the $\nu$-block $\mathbf M_{\nu\nu}$ is constrained by $\mathbf P_n$ and the
$\iota$-block is constrained by $\mathbf A_n$.

\paragraph{Step III-b: Component-wise identifiability of $\mathbf z_\nu^t$.}
We now show that $\mathbf M_{\nu\nu}$ must be a scaling-permutation matrix and that
$\mathbf M_{\nu\iota}=\mathbf 0$.
Fix any responsive coordinate $j\in\{1,\ldots,d_\nu\}$.
By Assumption~\ref{asm:ts_isolation}, there exists an environment $\mathbf u^{(j)}$ such that
\begin{equation}
z_{\nu,j}^{t}=c_j(\mathbf u^{(j)})+n_{\nu,j}^{t},
\label{eq:ts_isolated_j}
\end{equation}
i.e., the parent contribution to $z_{\nu,j}^{t+1}$ is removed.
Under this environment, Eq.~\eqref{eq:ts_h_affine} implies
\begin{equation}
z_{\nu,j}^{t}
=
\mathbf e_j^\top \mathbf M_{\nu\nu}\hat{\mathbf z}_\nu^{t}
+
\mathbf e_j^\top \mathbf M_{\nu\iota}\hat{\mathbf z}_\iota^{t}
+
c'_j,
\label{eq:ts_znu_j_affine}
\end{equation}
for some constant $c'_j$.
On the other hand, by Eq.~\eqref{eq:ts_isolated_j} and the innovation identifiability
in Eq.~\eqref{eq:ts_id_n_nu}, $z_{\nu,j}^{t}$ depends on exactly one independent Gaussian
innovation coordinate (up to an affine transform).
If either $\mathbf M_{\nu\nu}$ had more than one nonzero entry in its $j$-th row,
or $\mathbf M_{\nu\iota}\neq\mathbf 0$, then the right-hand side of
Eq.~\eqref{eq:ts_znu_j_affine} would depend on multiple independent Gaussian sources,
contradicting Eq.~\eqref{eq:ts_id_n_nu}.
Therefore, the $j$-th row of $\mathbf M_{\nu\nu}$ has exactly one nonzero entry and
$\mathbf M_{\nu\iota}=\mathbf 0$.
Since $j$ was arbitrary, $\mathbf M_{\nu\nu}$ must be a scaling-permutation matrix,
which establishes
\begin{equation}
\mathbf z_\nu^t=\mathbf P_\nu\hat{\mathbf z}_\nu^t+\mathbf c_\nu.
\label{eq:ts_id_nu_final}
\end{equation}

\paragraph{Step III-c: Linear block identifiability of $\mathbf z_\iota^t$.}
Finally, we consider the invariant block.
By Assumption~\ref{asm:ts_align}, the invariant representation $\hat f_\iota(\mathbf x^t)$
is independent of $\mathbf u$ and contains no $\nu$-information, which implies
$\mathbf M_{\iota\nu}=\mathbf 0$ in Eq.~\eqref{eq:ts_M_blocks}.
Together with the affine relation in Eq.~\eqref{eq:ts_h_affine} and the
block-wise affine identifiability of $\mathbf n_\iota^{t}$ in Eq.~\eqref{eq:ts_id_n_iota},
the same Gaussian-to-Gaussian argument as in Step~II-b implies that
$\mathbf M_{\iota\iota}$ must be nonsingular.
Hence there exist a nonsingular matrix $\mathbf A_\iota$ and a vector $\mathbf c_\iota$ such that
\begin{equation}
\mathbf z_\iota^t=\mathbf A_\iota\hat{\mathbf z}_\iota^t+\mathbf c_\iota,
\label{eq:ts_id_iota_final}
\end{equation}
which completes the proof.
\end{proof}

\section{Training Step Overview}
\label{app:alg_overview}

\Cref{alg:citevae_simple} summarizes one training step of \textsc{CITE-VAE}.
Given a perturbed pseudo-transition $\mathbf{x}_{t:t+1}^{(\mathbf{u})}$ and a matched control anchor
$\mathbf{x}_{t:t+1}^{(\mathbf{0})}$, we encode each snapshot into invariant and responsive latents,
and decode each observation using the emission model $p_\theta(\mathbf{x}^t\mid \mathbf{z}^t)$.
We then optimize a unified objective consisting of: (i) a reconstruction term,
(ii) a temporal KL regularizer induced by the interventional transition prior,
(iii) an $L_2$ alignment loss enforcing invariance across matched perturbed and control samples, and
(iv) an $\ell_1$ sparsity penalty on the triangular lagged adjacency parameters.

\begin{algorithm}[t]
\caption{\textsc{CITE-VAE}: One Training Step}
\label{alg:citevae_simple}
\DontPrintSemicolon
\SetKwInOut{Input}{Input}
\SetKwInOut{Output}{Output}
\SetKw{KwRet}{return}

{\footnotesize
\Input{OT-coupled mini-batch 
$\mathcal B=\{(\xb^{(u)}_{t:t+1},\xb^{(0)}_{t:t+1},\ub)\}$; 
weights $\lambda_{\mathrm{align}},\lambda_{\mathrm{reg}}$.}
\Output{Minimization loss $\mathcal L$.}

Encode snapshots into 
$\zb_s^{(c)}=(\zb_s^{\iota,(c)},\zb_s^{\nu,(c)})$ 
for $c\in\{u,0\}$ and $s\in\{t,t+1\}$\;

Construct condition-specific causal weights 
$\mathbf W^{(c)}$ from condition embedding $\ub^{(c)}$, and parameterize
the transition prior 
$p_\psi(\zb_{t+1}\mid \zb_t,\ub^{(c)})$\;

Compute the temporal ELBO $\mathcal L_{\mathrm{temp}}$, including
snapshot reconstruction, one-step predictive reconstruction, and temporal KL terms\;

$\mathcal L_{\mathrm{align}}
\leftarrow
\sum_{s\in\{t,t+1\}}
\|\mub^{\iota,(u)}_s-\mub^{\iota,(0)}_s\|_2^2$\;

$\mathcal L_{\mathrm{reg}}
\leftarrow
\sum_{c\in\{u,0\}}\|\mathbf W^{(c)}\|_1$\;

$\mathcal L
\leftarrow
-\mathcal L_{\mathrm{temp}}
+\lambda_{\mathrm{align}}\mathcal L_{\mathrm{align}}
+\lambda_{\mathrm{reg}}\mathcal L_{\mathrm{reg}}$\;

\KwRet $\mathcal L$\;
}
\end{algorithm}

\section{Experimental Details}
\subsection{Temporal Causal3DIdent Dataset}
\label{app:causal3d}

We follow the implementation of the temporal \textsc{Causal3DIdent} benchmark used in prior work~\citep{lippe2022citris, lippe2022causal} to generate synthetic sequences. All implementation choices are limited to the data generation process and do not affect the proposed model or its training procedure. The benchmark provides ground-truth latent causal variables with seven factors:
(i) object position as a three-dimensional vector $[x,y,z] \in [-2,2]^3$;
(ii) object rotation as two angles $[\alpha,\beta] \in [0,2\pi)^2$;
(iii) the hue of the object, background, and spotlight, each in $[0,2\pi)$;
(iv) spotlight rotation in $[0,2\pi)$; and
(v) object shape as a categorical variable with seven possible values.

Causal relations among the factors are predefined and time-invariant, yielding a first-order Markov process.
Each continuous factor follows a Gaussian transition model whose mean is a nonlinear function of its parent variables.
This construction induces strong dependencies among factors across time. Interventions are encoded by a binary vector $\mathbf{u}\in\{0,1\}^K$ indicating which factors are targeted.
In our setup, $\mathbf{u}$ remains constant within each sequence.

\subsubsection{Analogy between synthetic interventions and single-cell perturbations.}
\label{app:analogy}
\Cref{app:synthetic} summarizes the conceptual correspondence between the synthetic 3D intervention setting and single-cell perturbation experiments.

\begin{table}[h]
\centering
\small
\begin{tabular}{ll}
\toprule
\textbf{Synthetic 3D environment} & \textbf{Single-cell perturbation} \\
\midrule
Single rendered object & Single cell \\
Object shape (fixed) & Cellular context \\
Object position / rotation & Perturbation-responsive gene program \\
Controlled factor intervention & Perturbation \\
Invariant visual attributes & Background cellular program \\
Temporal factor evolution & Post-perturbation response trajectory \\
\bottomrule
\end{tabular}
\caption{Conceptual analogy between the synthetic benchmark and single-cell perturbation experiments.}
\label{app:synthetic}
\end{table}

\subsubsection{Evaluation Protocol for Manifold Recovery.}
To quantify how well \textsc{CITE-VAE} recovers the ground-truth generative factors, we adopt a standard nonlinear-probe evaluation.
Specifically, we fit a lightweight MLP regressor that maps the learned latent blocks $(\hat{\mathbf z}_\nu,\hat{\mathbf z}_\iota)$ to the ground-truth factor space.
Before fitting, both the latent representations and ground-truth factors are standardized using a global affine transform computed on the training split.
The probe is trained post hoc and evaluated on a held-out test set; no gradients are propagated through the representation model.
We report the coefficient of determination ($R^2$) and Spearman rank correlation between predicted and true factors.
This protocol measures the decodability and geometric fidelity of the learned representation while being insensitive to benign invertible reparameterizations of the latent space.

\subsection{hiPSC Dataset}
\label{app:hispc}
\subsubsection{Longitudinal Single-Cell Perturbations.}
We use the longitudinal hiPSC single-cell CRISPR perturbation dataset~\citep{ishikawa2023renge}, containing 25,293 cells measured at four discrete time points ($t\in\{2,3,4,5\}$) under 23 single-gene knockouts plus a shared control.
Control cells are identified by the provided ($N_{\text{ctrl}}=1, 655$).

\subsubsection{Preprocessing.}
We apply library-size normalization (10k counts per cell) followed by a \texttt{log1p} transformation. Unless otherwise specified, we restrict training to the top 1,000 highly variable genes (HVGs). To avoid feature-selection leakage, HVGs are selected using training data only and the same gene set is used for validation and test. Intervention labels are obtained from \texttt{target\_gene}, and cells with missing intervention annotations are excluded.

\subsubsection{Discrete-time transition construction (independent coupling).}
Because continuous single-cell trajectories are unobserved, we approximate discrete-time transitions using an \emph{independent coupling} between adjacent snapshots. For each intervention condition and adjacent time pair ($t\!\to\!t{+}1$), we randomly sample cells from time $t$ and time $t{+}1$ to form training tuples $(x_t, x_{t+1}, u)$. This construction trains the model to map the population distribution at time $t$ to that at time $t{+}1$ without relying on any ground-truth cell-to-cell correspondence. Random sampling is performed with a fixed seed for reproducibility.

\subsubsection{Leave-One-Intervention-Out (LOO)}
We evaluate OOD generalization via leave-one-intervention-out: for each held-out KO gene $k$, we train on all remaining perturbations (including the shared control) and test exclusively on $k$, reporting results averaged over all held-out genes. When evaluating on DE gene subsets, DE genes are computed \emph{within each LOO training fold only} using training interventions and control, strictly excluding the held-out intervention to prevent leakage.

\subsection{Metrics.}
\label{app:metrics}

Metrics are computed on \emph{pseudobulk} profiles by averaging cells within each perturbation--time condition $(\mathbf u,t)$.
Pseudobulk evaluation is standard in perturbation prediction and provides a comparable target across methods that produce condition-level outputs.
In our setting, it additionally tests whether a cell-level model can faithfully capture population-level perturbation effects.

\subsubsection{$R^2$}
On real single-cell datasets, we report $R^2$ at the pseudobulk level.
For each condition $(\mathbf u,t)$, let $\mathcal I_{\mathbf u,t}$ denote the index set of observed cells.
We form the mean observed expression
\begin{equation}
\bar{\mathbf x}^{t,(\mathbf u)}
=
\frac{1}{|\mathcal I_{\mathbf u,t}|}
\sum_{i\in\mathcal I_{\mathbf u,t}}
\mathbf x^{t,(\mathbf u),(i)},
\end{equation}
and the mean predicted expression
\begin{equation}
\bar{\hat{\mathbf x}}^{t,(\mathbf u)}
=
\frac{1}{|\mathcal I_{\mathbf u,t}|}
\sum_{i\in\mathcal I_{\mathbf u,t}}
\hat{\mathbf x}^{t,(\mathbf u),(i)}.
\end{equation}
Let $\mathcal C$ be the set of test conditions $(\mathbf u,t)$. We compute
\begin{equation}
R^2
=
1
-
\frac{
\sum_{(\mathbf u,t)\in\mathcal C}
\left\|
\bar{\hat{\mathbf x}}^{t,(\mathbf u)}
-
\bar{\mathbf x}^{t,(\mathbf u)}
\right\|_2^2
}{
\sum_{(\mathbf u,t)\in\mathcal C}
\left\|
\bar{\mathbf x}^{t,(\mathbf u)}
-
\bar{\bar{\mathbf x}}
\right\|_2^2
},
\end{equation}
where
\begin{equation}
\bar{\bar{\mathbf x}}
=
\frac{1}{|\mathcal C|}
\sum_{(\mathbf u,t)\in\mathcal C}
\bar{\mathbf x}^{t,(\mathbf u)}
\end{equation}
denotes the global mean pseudobulk expression across test conditions.

\subsubsection{$\Delta$ Pearson}
Let $\bar{\mathbf x}^{t,(\mathbf 0)}$ denote the mean control pseudobulk expression at time $t$.
For each cell $i$ under condition $(\mathbf u,t)$, we define the control-centered shifts
\begin{equation}
\Delta \mathbf x^{t,(\mathbf u),(i)}
=
\mathbf x^{t,(\mathbf u),(i)}
-
\bar{\mathbf x}^{t,(\mathbf 0)},
\qquad
\Delta \hat{\mathbf x}^{t,(\mathbf u),(i)}
=
\hat{\mathbf x}^{t,(\mathbf u),(i)}
-
\bar{\mathbf x}^{t,(\mathbf 0)}.
\end{equation}
We compute the cell-level Pearson correlation
\begin{equation}
\rho^{t,(\mathbf u),(i)}
=
\mathrm{corr}\!\left(
\Delta \hat{\mathbf x}^{t,(\mathbf u),(i)},
\Delta \mathbf x^{t,(\mathbf u),(i)}
\right),
\end{equation}
and report
\begin{equation}
\Delta\mathrm{Pearson}
=
\mathbb E_{(\mathbf u,t)}
\mathbb E_{i\sim \mathrm{Unif}(\mathcal I_{\mathbf u,t})}
\!\left[
\rho^{t,(\mathbf u),(i)}
\right].
\end{equation}

\subsubsection{MAE}
For each condition $(\mathbf u,t)$, let $\bar{\mathbf x}^{t,(\mathbf u)}$ and $\bar{\hat{\mathbf x}}^{t,(\mathbf u)}$
denote the observed and predicted pseudobulk profiles as defined above.
We define the control-centered perturbation effects
\begin{equation}
\Delta\bar{\mathbf x}^{t,(\mathbf u)}
=
\bar{\mathbf x}^{t,(\mathbf u)}
-
\bar{\mathbf x}^{t,(\mathbf 0)},
\qquad
\Delta\bar{\hat{\mathbf x}}^{t,(\mathbf u)}
=
\bar{\hat{\mathbf x}}^{t,(\mathbf u)}
-
\bar{\mathbf x}^{t,(\mathbf 0)}.
\end{equation}
We compute condition-level MAE across genes:
\begin{equation}
\mathrm{MAE}^{t,(\mathbf u)}
=
\frac{1}{G}
\sum_{g=1}^{G}
\left|
\Delta\bar{\hat{x}}^{t,(\mathbf u)}_{g}
-
\Delta\bar{x}^{t,(\mathbf u)}_{g}
\right|,
\end{equation}
and report the average of $\mathrm{MAE}^{t,(\mathbf u)}$ over test conditions (and timepoints, when applicable).

\paragraph{DE Gene Identification (AUC-ROC, AUPRC)}
Let $\bar{\mathbf x}^{t,(\mathbf u)}$ and $\bar{\mathbf x}^{t,(\mathbf 0)}$ denote the observed pseudobulk profiles
under perturbation $\mathbf u$ and control at time $t$, respectively.
We define the ground-truth absolute perturbation effect
\begin{equation}
\boldsymbol\delta^{\mathrm{true}}
=
\left|
\bar{\mathbf x}^{t,(\mathbf u)}
-
\bar{\mathbf x}^{t,(\mathbf 0)}
\right|
\in\mathbb R^{G}.
\end{equation}
We label genes as differentially expressed (DE) by selecting the top-$K$ genes with the largest entries in
$\boldsymbol\delta^{\mathrm{true}}$:
\begin{equation}
y_g
=
\mathbb I\!\left[
\delta^{\mathrm{true}}_g
\ge
\mathrm{Top}\text{-}K\!\left(\boldsymbol\delta^{\mathrm{true}}\right)
\right],
\qquad g=1,\dots,G.
\end{equation}
Given predicted pseudobulk profiles $\bar{\hat{\mathbf x}}^{t,(\mathbf u)}$, we compute the predicted absolute perturbation effect
\begin{equation}
\boldsymbol\delta^{\mathrm{pred}}
=
\left|
\bar{\hat{\mathbf x}}^{t,(\mathbf u)}
-
\bar{\mathbf x}^{t,(\mathbf 0)}
\right|.
\end{equation}
Treating $y_g$ as binary labels and $\delta^{\mathrm{pred}}_g$ as confidence scores, we report:
\begin{itemize}
\item \textbf{AUC-ROC}: area under the ROC curve for distinguishing DE vs.\ non-DE genes;
\item \textbf{AUPRC}: area under the precision--recall curve, reflecting performance under class imbalance.
\end{itemize}
Both metrics are computed at the gene level and averaged across test perturbations
(and timepoints, when applicable).

\subsection{Ablation on alignment and sparsity regularization.}

Our full objective combines the temporal loss with two auxiliary components, namely an alignment term and a sparsity regularization term. To evaluate the contribution of these components, we perform a controlled ablation study in which we remove the alignment term, the sparsity term, or both, while keeping the data processing, model architecture, and training protocol fixed. This allows us to test whether the performance gains of \textsc{CITE-VAE} arise from the temporal objective alone or from the joint effect of temporal modeling, alignment, and sparsity regularization.
\begin{table}[!t]
\centering
\scriptsize
\setlength{\tabcolsep}{3pt}
\renewcommand{\arraystretch}{1.08}
\caption{Ablation of the alignment and sparsity regularizers on the hiPSC benchmark.}
\label{tab:ablation_hiPSC}

\resizebox{\columnwidth}{!}{
\begin{tabular}{lccc|ccc}
\toprule
Setting
& \multicolumn{3}{c|}{All genes}
& \multicolumn{3}{c}{DE genes} \\
\cmidrule(lr){2-4}
\cmidrule(lr){5-7}
& RMSE $\downarrow$ & $R^2 \uparrow$ & MAE $\downarrow$
& $\Delta$ Pearson $\uparrow$ & AUC $\uparrow$ & AUPRC $\uparrow$ \\
\midrule
\textsc{\textbf{CITE-VAE}}
& \textbf{0.086$\pm$0.054} & $>0.95$ & \textbf{0.033$\pm$0.015}
& \textbf{0.691$\pm$0.007} & 0.931$\pm$0.020 & \textbf{0.617$\pm$0.061} \\

\textsc{w/o alignment}
& 0.116$\pm$0.043 & $>0.95$ & 0.048$\pm$0.014
& 0.389$\pm$0.014 & 0.938$\pm$0.017 & 0.547$\pm$0.063 \\

\textsc{w/o sparsity}
& 0.111$\pm$0.043 & $>0.95$ & 0.046$\pm$0.013
& 0.399$\pm$0.016 & \textbf{0.943$\pm$0.012} & 0.559$\pm$0.059 \\

\textsc{w/o align. \& sparse.}
& 0.108$\pm$0.043 & $>0.95$ & 0.044$\pm$0.013
& 0.385$\pm$0.015 & 0.940$\pm$0.015 & 0.537$\pm$0.052 \\
\bottomrule
\end{tabular}
}

\end{table}

Table~\ref{tab:ablation_hiPSC} shows that the full \textsc{CITE-VAE} objective consistently performs best, with the clearest gains appearing on DE-focused metrics such as $\Delta$ Pearson and AUPRC, which are most sensitive to perturbation-responsive effects. In contrast, removing alignment, sparsity, or both leads to very similar performance across most metrics, especially on the DE-related evaluations. This pattern suggests that the benefit of these two terms is not simply additive.

We do not interpret alignment and sparsity as two independent performance boosters. Rather, their interaction is induced by the coupled latent architecture of \textsc{CITE-VAE}. In our model, the responsive state $\mathbf{z}^t_\nu$ is explicitly conditioned on the invariant state $\mathbf{z}^t_\iota$ through the latent transition mechanism. As a result, alignment and sparsity act on different but interdependent aspects of the same latent causal system: the alignment term stabilizes the separation between invariant and responsive factors, while the sparsity term constrains the responsive transition mechanism learned on top of that separation. Under this hierarchical structure, removing either one is already sufficient to drive optimization toward a similar suboptimal regime, which explains why the three ablated variants nearly overlap.

Overall, the ablation supports the view that the two regularizers play a cooperative, non-additive role. Their main effect is not to provide separable linear improvements, but to jointly stabilize the latent causal decomposition required for robust OOD perturbation prediction.

Interestingly, the difference is much less pronounced on coarse global metrics such as pseudobulk $R^2$, but becomes substantial on DE-oriented metrics. This further suggests that the main role of alignment and sparsity is not merely to improve global reconstruction fidelity, but to preserve the structure of perturbation-responsive mechanisms.

\subsection{Auxiliary benchmark: E-MTAB-14065.}

\begin{table*}[t]
\caption{Results on \texttt{E-MTAB-14065} (LOO mode  ). Metrics are reported as mean $\pm$ standard deviation across 3 random seeds.}
\centering
\scriptsize
\setlength{\tabcolsep}{2pt}
\renewcommand{\arraystretch}{1.05}
\begin{adjustbox}{width=0.75\textwidth,center}
\begin{tabular}{l|ccc|ccc}
\toprule
Model
& \multicolumn{3}{c|}{All genes}
& \multicolumn{3}{c}{DE genes} \\
\cmidrule(lr){2-4}
\cmidrule(lr){5-7}
& RMSE $\downarrow$ & $R^2 \uparrow$ & MAE $\downarrow$
& $\Delta$ Pearson $\uparrow$ & AUC-ROC $\uparrow$ & AUPRC $\uparrow$ \\
\midrule
\textbf{ElasticNet}~\citep{zou2005regularization}
& 0.7075$\pm$0.0000 & 0.7212$\pm$0.0000 & 0.5502$\pm$0.0000
& 0.7354$\pm$0.0003 & 0.8662$\pm$0.0012 & 0.3846$\pm$0.0008 \\

\textbf{scGPT}~\citep{cui2024scgpt}
& 0.6660$\pm$0.0315 & 0.9253$\pm$0.0649 & 0.5029$\pm$0.0230
& 0.7199$\pm$0.0634 & 0.8387$\pm$0.0083 & 0.2997$\pm$0.0494 \\

\textbf{UCE}~\citep{rosen2023universal}
& - & - & -
& 0.4630$\pm$0.0307 & 0.8067$\pm$0.0182 & 0.2058$\pm$0.0199 \\

\textbf{SAMS-VAE}~\citep{bereket2023modelling}
& 0.6723$\pm$0.0016 & 0.9322$\pm$0.0089 & 0.4892$\pm$0.0010
& 0.7531$\pm$0.0067 & 0.7970$\pm$0.0039 & 0.1698$\pm$0.0040 \\

\textbf{DiscrepancyVAE}~\citep{zhang2023identifiability}
& 0.8902$\pm$0.0135 & 0.8236$\pm$0.0120 & 0.7102$\pm$0.0105
& 0.3287$\pm$0.0569 & 0.7098$\pm$0.0123 & 0.2063$\pm$0.0136 \\

\textbf{RENGE}~\citep{ishikawa2023renge}
& 0.7724$\pm$0.0004 & 0.8601$\pm$0.0019 & 0.8070$\pm$0.0017
& 0.2496$\pm$0.0227 & 0.8787$\pm$0.0231 & 0.9944$\pm$0.0010 \\

\textbf{scNODE}~\citep{zhang2024scnode}
& 0.6719$\pm$0.0025 &  0.9718$\pm$0.0039 & 0.4849$\pm$0.0022
& 0.3428$\pm$0.4189 & 0.7845$\pm$0.0317 & 0.2640$\pm$0.0580 \\

\textbf{CellOT}~\citep{bunne2023learning}
& - & - & -
& 0.0784$\pm$0.0603 & 0.5616$\pm$0.0490 & 0.0710$\pm$0.0162 \\

\textbf{LFADS}~\citep{pandarinath2018inferring}
& 0.6347$\pm$0.0023 &  0.9583$\pm$0.0082 & 0.4644$\pm$0.0007
& 0.7511$\pm$0.0138 & 0.7454$\pm$0.0069 & 0.1956$\pm$0.0112 \\

\textsc{\textbf{CITE-VAE}} (ours)
& 0.6346$\pm$0.0014 &  0.9852$\pm$0.004 & 0.4712$\pm$0.0001
& 0.7510$\pm$0.0057 & 0.8303$\pm$0.0130 & 0.3061$\pm$0.0223 \\
\bottomrule
\end{tabular}
\end{adjustbox}
\label{tab:emtab14065_full_eval}
\end{table*}

We additionally evaluate our method on \texttt{E-MTAB-14065} (iPS2-10X monolayer cardiac differentiation), an inducible perturbation single-cell RNA-seq dataset generated with the iPS2-seq platform in a monolayer hiPSC cardiac differentiation system~\citep{balmas2025single}. The dataset provides measurements across discrete developmental stages (\texttt{D0}, \texttt{D2}, \texttt{D6}, and \texttt{D12}, with an additional \texttt{D23} sample available separately), together with perturbation annotations and clone identities. Compared with larger perturbation benchmarks, \texttt{E-MTAB-14065} covers fewer perturbations but offers a more structured temporal and clonal setting. We therefore use it as an auxiliary temporal benchmark for evaluating whether a model can capture perturbation-responsive developmental dynamics under substantial clone-level variation.

For preprocessing, we restrict the analysis to the main developmental time-course experiment (\texttt{exp7}) and retain the four core stages \texttt{D0}, \texttt{D2}, \texttt{D6}, and \texttt{D12}. Perturbations are represented at the gene level, and we use \texttt{SCR} as the sole control condition. We exclude \texttt{B2M} from the main evaluation protocol to maintain a single reference condition across methods. Starting from the author-provided filtered metadata and raw count matrix, we reconstruct the benchmark by aligning cell-level annotations with raw counts, while preserving clone and batch annotations. To improve the stability of temporal matching and downstream aggregate statistics, we remove highly sparse \texttt{time $\times$ perturbation $\times$ clone} strata with fewer than 20 cells.

The resulting benchmark contains 24,766 cells from 41 clones, including 23,410 perturbed cells and 1,356 control cells. The retained perturbation set consists of five target genes (\texttt{CHD7}, \texttt{GATA4}, \texttt{KMT2D}, \texttt{NKX2.5}, and \texttt{SMAD2}) together with the \texttt{SCR} control. For evaluation, we provide two aligned expression views: an all-gene matrix of size \mbox{$24{,}766 \times 17{,}780$} and an HVG-based matrix of size \mbox{$24{,}766 \times 2{,}003$}, where the HVG panel includes the selected highly variable genes together with the perturbed condition genes. For differential-expression-based evaluation, we further construct time-matched DE reference sets against \texttt{SCR}; the union DE set contains 729 genes under the top-100 protocol.

On this auxiliary benchmark in \Cref{tab:emtab14065_full_eval}, \textsc{CITE-VAE} remains competitive across both transcriptome-level reconstruction and DE-oriented metrics, even though \texttt{E-MTAB-14065} differs substantially from our main benchmark, with fewer perturbations, stronger clone-level variation, and a more structured developmental time-course design. In particular, \textsc{CITE-VAE} achieves the best RMSE among the stable baselines and remains among the strongest models on $R^2$, while maintaining low MAE. This indicates that the model transfers well to clone-aware developmental perturbation data and does not depend on a specific benchmark structure.

For DE-based evaluation, \textsc{CITE-VAE} achieves strong $\Delta$ Pearson and competitive AUC-ROC/AUPRC, although it is not uniformly best on every metric. This pattern is expected: some baselines, especially linear or ranking-oriented ones such as ElasticNet and RENGE, can perform well on discrimination-style DE metrics that emphasize gene ranking, while being less competitive in full-transcriptome reconstruction. By contrast, \textsc{CITE-VAE} provides a more balanced trade-off between global expression fidelity and perturbation-effect consistency.

We also note that some baselines were less stable in this setting. In particular, UCE and CellOT did not yield reliable transcriptome-level reconstruction metrics under our evaluation protocol, suggesting a mismatch between their modeling assumptions and the relatively small, clone-structured temporal perturbation regime of \texttt{E-MTAB-14065}. Overall, these results support the use of \texttt{E-MTAB-14065} as an auxiliary temporal benchmark and further suggest that \textsc{CITE-VAE} generalizes beyond the main dataset while preserving strong overall performance.

\end{document}